\documentclass[lettersize,journal]{IEEEtran}
\usepackage{algorithm}
\usepackage{array}
\usepackage{textcomp}
\usepackage{stfloats}
\usepackage{url}
\usepackage{verbatim}
\IEEEoverridecommandlockouts                              
\usepackage[utf8]{inputenc}
\usepackage[noadjust]{cite}
\usepackage{graphicx}       
\usepackage[T1]{fontenc}    
\usepackage{algorithm}      
\usepackage[noend]{algpseudocode}
\algrenewcommand\algorithmicindent{0.7em}
\usepackage{ifthen}         
\usepackage{xspace}         
\usepackage{hyperref}       
\hypersetup{
    colorlinks=true,
    linkcolor=blue,
    filecolor=magenta,      
    urlcolor=cyan,
}
\usepackage{microtype}

\usepackage{svg}
\usepackage{amsmath,amsfonts}
\usepackage[noabbrev]{cleveref}
\usepackage[utf8]{inputenc}
\usepackage{mathtools}
\usepackage{cleveref}
\usepackage{pdfpages}
\usepackage{amssymb}

\usepackage{tabularx}
\usepackage{bm}
\usepackage{pdfescape}
\graphicspath{{assets/}} 
\newcolumntype{M}[1]{>{\centering\arraybackslash}m{#1}}

\newcommand{\myproof}[1]{%
  \noindent\hspace{2em}{\itshape Proof of #1: }%
}

\newcommand{\myendproof}{\hfill\ensuremath{\blacksquare}\par\endtrivlist\unskip}

\newcommand\irisuu{\algname{IRIS-U$^2$}}

\newcommand\Nloc{\mathcal{N}_{\rm unc}\xspace}
\newcommand\Dloc{\mathcal{D}_{\rm unc}\xspace}

\newcommand{\ek}{{\ensuremath{(\eps,\kappa)}}\xspace}

\newcommand{\colProbHat}{\hat{\mathcal{C}}(\pi)}
\newcommand{\colProbBar}{\bar{\mathcal{C}}(\pi)}

\newcommand\expL{\bar{\ell}(\pi)\xspace}
\newcommand\expS{\bar{\mathcal{S}}(\pi)\xspace}
\newcommand\estL{\hat{\ell}(\pi)\xspace}
\newcommand\estS{\hat{\mathcal{S}}(\pi)\xspace}

\newcommand\Mexample{\mathcal{M}_{\text{toy}}}
\newcommand\Msimple{\mathcal{M}_{\text{simple}}}
\newcommand\MsimpleBridge{\hat{\mathcal{M}}_{\text{Simple-GNSS-INS}}}
\newcommand\Madvanced{\mathcal{M}_{\text{advanced}}}
\newcommand\Msimplified{\mathcal{M}_{\text{simplified}}}
\newcommand\Mexact{\mathcal{M}_{\text{exact}}}

\usepackage{caption}
\usepackage{subcaption}
\newtheorem{prob}{Problem}
\newtheorem{dfn}{Definition}

\newtheorem{theorem}{thm}[section]
\newtheorem{lemma}[theorem]{Lemma}

\newtheorem{assumption}{Assumption}

\def\kiril#1{\textcolor{blue}{(\textbf{KS:} #1)}}

\def\check#1{{ #1}}

\newcommand{\ignore}[1]{}
\newcommand{\ignoreStat}[1]{}

 \def\I{\mathcal{I}} \def\E{\mathcal{E}}
\def\S{\mathcal{S}} \def\G{\mathcal{G}} 
\def\I{\mathcal{I}}  
 \def\V{\mathcal{V}} 
  
 \def\X{\mathcal{X}} 
 
\def\R{\mathcal{I}}

\def\eps{\varepsilon}

\newcommand{\Cpp}{C\raise.08ex\hbox{\tt ++}\xspace}

\newcommand\algname[1]{\textsf{#1}\xspace}

\newcommand\astar{\algname{A*}}

\newcommand\iris{\algname{IRIS}}

\newcommand\lum{\algname{UP-\iris}}

\newcommand\ap{\algname{AP}}
\newcommand\pap{\algname{PAP}}
\newcommand\paps{{{\pap}s}\xspace}
\newcommand\pp{\algname{PP}}

\newboolean{HIDENOTES}
\setboolean{HIDENOTES}{false}

\ifthenelse{\boolean{HIDENOTES}}
    {
\newcommand{\MF}[1]{{}}
\newcommand{\OS}[1]{{}}
\newcommand{\RA}[1]{{}}
    }
    {
\newcommand{\MF}[1]{\textcolor{purple}{\textbf{[MF]:}~#1}}
\newcommand{\OS}[1]{\textcolor{orange}{\textbf{[OS]:}~#1}}
\newcommand{\RA}[1]{{\textcolor{blue}{\textbf{[RA]:}~#1}}}
    }

\renewcommand{\vec}{\mathbf}

\colorlet{pink}{red!40}
\colorlet{light_blue}{cyan!60}

\newboolean{POI}
\newboolean{ROI}

\setboolean{POI}{true}
\ifthenelse{\boolean{POI}}
	{\setboolean{ROI}{false}}
	{\setboolean{ROI}{true}}

\newcommand{\inspectionType}[2]
{\ifthenelse{\boolean{POI} }{{#1}}{}\ifthenelse{\boolean{ROI}}{#2}{}}

\title{\LARGE \bf
Inspection planning under execution uncertainty
}

\author{Shmuel David Alpert, Kiril Solovey, Itzik Klein, Oren Salzman}

\date{September 2022}

\begin{document}

\maketitle

\begin{abstract}%
\label{Sec:abstract}%
Autonomous inspection tasks necessitate path-planning algorithms to efficiently gather observations from \emph{points of interest}~(POI). However, localization errors commonly encountered in urban environments can introduce execution uncertainty, posing challenges to successfully completing such tasks. 
\check{Unfortunately, existing algorithms for inspection planning do not explicitly account for execution uncertainty, which can hinder their performance. To bridge this gap, we present \emph{~\iris-under uncertainty}~(\irisuu), the first inspection-planning algorithm that offers statistical guarantees regarding coverage, path length, and collision probability. 
Our approach builds upon \iris---our framework for \emph{deterministic} inspection planning, which is highly efficient and provably asymptotically-optimal. The extension to the much more involved uncertain setting is achieved by a refined search procedure that estimates POI coverage probabilities using Monte Carlo (MC) sampling.}
The efficacy of~\irisuu is demonstrated through a case study focusing on structural inspections of bridges. Our approach exhibits improved expected coverage, reduced collision probability, and yields increasingly precise \check{statistical guarantees} as the number of MC samples grows. Furthermore, we demonstrate the potential advantages of computing bounded sub-optimal solutions to reduce computation time while maintaining statistical guarantees.

\end{abstract}



\section{Introduction}
\label{Sec:Introduction and related work}
We consider the problem of planning in an offline phase a collision-free path for a robot to inspect a set of \emph{points of interest}~(POIs) using onboard sensors.
This can be challenging, especially in urban environments, where dynamics uncertainty and localization errors (e.g., inaccuracies in location estimates) increase the task's complexity. In particular, localization can be a significant source of uncertainty in urban environments, leading to missed POIs and compromising the efficiency and accuracy of inspection missions.

One application that motivates this work is the inspection of bridges using \emph{unmanned aerial vehicles}~(UAVs)~\cite{bircher2016three}. Almost~$40\%$ of the bridges in the United States of America exceed their~$50$-year design life~\cite{mcguire2016bridge}, and regular inspections are critical to ensuring bridge safety. UAVs can efficiently inspect bridge structures via visual assessment at close range without involving human inspectors or expensive  under-bridge inspection units~\cite{chan2015towards}.
In these scenarios, the UAV typically carries a camera for POI inspection and a navigation system that combines data from various sensors like a \emph{global navigation satellite system}~(GNSS) and inertial sensors. Yet, GNSS signal obstruction by the bridge can cause location inaccuracies, heavily relying on inertial measurements (which also suffer from inequaricies) and potentially compromising inspection effectiveness (see Fig.~\ref{fig:figure12}).

\begin{figure}[t!]
    \centering
    \includegraphics[width=0.5\textwidth]{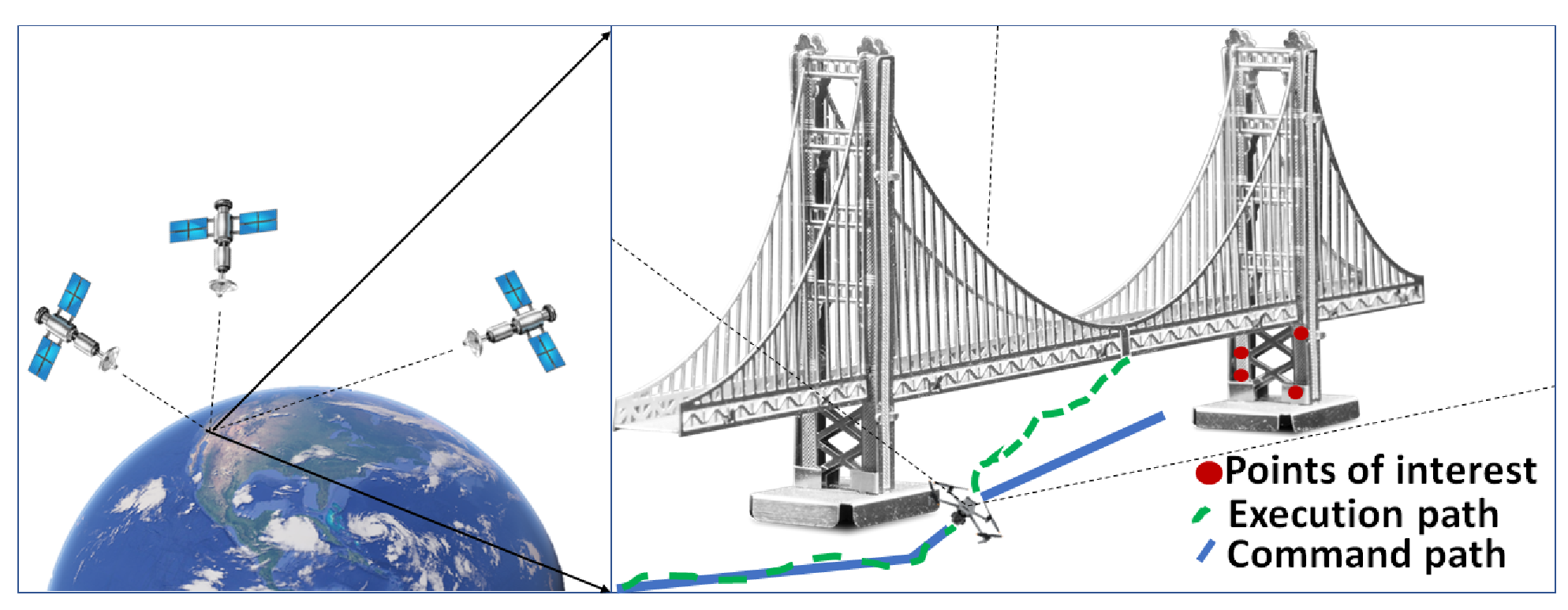}
    \caption{
    A UAV executing a bridge-inspection task, following a command path pre-computed offline (dark blue). 
    The UAV relies on GNSS satellite signals for navigation but the bridge obstructs one satellite signal during a particular segment of its path (light blue).
    This leads to a substantial deviation from the command path, which hinders the UAV's ability to inspect the  POIs and may result in a collision with the bridge.}
    \label{fig:figure12}
    \vspace{-4mm}
\end{figure}


Several approaches have been proposed to perform inspection planning without accounting for uncertainty~\cite{bircher2016three,fu2019toward,bircher2017incremental}. 
These methods can be used when the uncertainty is low by compensating for lack of GNSS by upgrading the navigation system to use 
RF tools such as  WIFI~\cite{sun2014wifi}, Bluetooth~\cite{rida2015indoor}
or  tactical-grade inertial sensors~\cite{klein2011vehicle}. 
However, these solutions require a supportive communication infrastructure or more expensive and heavy sensors.

\check{
As we detail in Sec.~\ref{sec:related-work}, uncertainty has been extensively considered when planning the motion of a robot without accounting for inspection considerations.
This has been done using various methods that allow to account for position uncertainty by e.g., using  
Monte Carlo localization~\cite{JansonSP15}, 
particle filters~\cite{MelchiorS07}, or 
Bayesian filtering to estimate and propagate uncertainty in the robot's position and environment~\cite{BergPA12}.
However, uncertainty in inspection tasks introduces further complications as it requires reasoning about an exceptionally large search space that captures the interaction between the POI locations, the locations from which POIs can be inspected, and the order of visitation of the latter locations. As a result, the computational complexity of inspection planning can be large, even in the absence of localization uncertainty~\cite{bircher2017incremental,DBLP:journals/ijrr/FuKSA23}.} 

\check{In our previous work, we introduced the \iris algorithm,  a highly effective approach for deterministic inspection planning~\cite{DBLP:journals/ijrr/FuKSA23}. Importantly, \iris guarantees asymptotic convergence to optimal solutions while being several orders of magnitude faster than previous work providing similar guarantees. }

\check{Unfortunately, extending \iris to the uncertain setting is highly nontrivial. Roughly speaking,  the efficiency of \iris is due to a novel graph-search algorithm that searches in the space of possible paths, where the search is guided by reasoning about which POIs have been seen along any given path. Unfortunately, in our context, execution uncertainty can lead to the robot deviating from its proposed path, which can lead to inspecting a different set of POIs from what it initially intended. 
%
}

\vspace{5pt}
\noindent \textbf{Contribution.}
\check{We present \iris-under uncertainty ({\irisuu})---the first algorithm for offline inspection planning that systematically accounts for execution uncertainty. Our algorithm combines the capabilities of (i) \iris, which efficiently explores the space of high-quality inspection plans in the deterministic setting, together with (ii) \emph{Monte Carlo}~(MC) sampling to reason about uncertainty via POI inspection probabilities. Importantly, our approach does not merely utilize those two components in a decoupled manner, by, e.g., interleaving between planning and uncertainty estimation~\cite{papachristos2019autonomous}, but rather uses the uncertainty model to (i) obtain statistical bounds (on the number of POIs inspected by each path, collision probability and length) and (ii) uses these statistical bounds to guide the exploration of the search space. 
To compute and integrate those probabilities within a search-based approach in a computationally efficient manner, we develop novel mechanisms for extending, subsuming, and dominating nodes.  The computed probabilities then serve as optimization objectives for the algorithm.}

On the theoretical side, \irisuu estimates with a certain \emph{confidence level}~(CL) the desired performance criteria (i.e., path length, coverage, and collision probabilities)  within some \emph{confidence interval}~(CI) depending on the number of the MC samples used~\cite{hazra2017using}. 
However, when selecting an execution path from multiple optional paths, straightforward statistical analysis may be biased toward false negatives and either requires offline prepossessing to provide guarantees or should be used as a guideline (see details in Appendix~\ref{subsec:Illustrative example for possible false negatives}). 
We choose the latter approach and outline a procedure to set a user-defined parameter for providing a CI, which becomes tighter as the number of samples increases.
Additionally, we highlight the potential benefits of using a bounded sub-optimal solution in certain situations to reduce computation time while still providing guarantees through the CI boundaries.

We demonstrate the effectiveness of \irisuu through a simulated case study of structural inspections of bridges using a UAV in an urban environment. Our results show that \irisuu is able to achieve a desired level of coverage, while also reducing collision probability and tightening the CI lower boundaries as the number of MC samples increases.

The rest of this paper is organized as follows. 
In Sec.~\ref{sec:related-work}, we review relevant related work.
In Sec.~\ref{Sec:ProblemDefinition}, we formulate the problem of offline path planning for inspection tasks under execution uncertainty. In Sec.~\ref{Sec:Algorithmic background}, we describe the algorithmic background and in Sec.~\ref{sec:method} we describe our proposed approach followed by a theoretical analysis presented in Sec.~\ref{sec:Theoretical guarantees}. 
Then in Sec.~\ref{Sec:DemoScenario} and~\ref{Sec:BridgeScenario} we present the results of our simulated experiments. Finally, in Sec.~\ref{Sec:Conclusion and future work}, we summarize our main contributions and discuss future work.

\section{Related Work}
\label{sec:related-work}
\subsection{Motion planning under uncertainty}
Various approaches have been proposed to address localization uncertainty for safe motion planning (which only considers a motion plan between states without inspection points).
For instance, minimum-distance collision-free paths can be computed while bounding the uncertainty by, e.g., a self-localization error ellipsoid, or uncertainty corridor~\cite{pepy2006safe, alami1994planning, candido2010minimum}. 
Another approach~\cite{delamer2019solving}, uses a mixed-observability Markov decision process approach to account for a-priori probabilistic sensor availability and path execution error propagation.
Englot et al.~\cite{englot2016sampling} suggest an RRT*-based method which minimizes the maximum uncertainty of any segment along a path. Sampling-based planners that explicitly reason about the evolution of the robot's belief have also been proposed~\cite{Wu.ea.22,zheng2023ibbt,DBLP:conf/icra/HoSL22,DBLP:journals/trob/PedramFT23}.

\subsection{Inspection planning}
Several approaches have been proposed for \emph{offline} inspection planning without accounting for uncertainty. Some algorithms decompose the region containing the POIs into sub-regions and solve for each sub-region separately~\cite{galceran2013survey}, while others divide the problem into two NP-hard problems:~(i)~solving the art gallery problem to compute a small set of viewpoints that collectively observe all the POIs, and~(ii)~solving the traveling salesman problem to compute the shortest path to visit this set of viewpoints~\cite{danner2000randomized, englot2010inspection, englot2012sampling, papachristos2019autonomous}. Another common approach is to simultaneously compute both inspection points (i.e., points from which POIs are inspected) and the trajectory to visit these inspection points using sampling-based techniques~\cite{papadopoulos2013asymptotically, bircher2017incremental}. However, these approaches do not consider the execution uncertainty, which can lead to missed POIs due to differences between the path executed and the path planned.

Alternative approaches consider the \emph{online} setting where a planner has to run online, deciding on a next step as the exploration or inspection advances~\cite{BircherKAOS18,PapachristosMKD19}.
Bircher et al.~\cite{BircherKAOS18} consider the setting where no uncertainty exists. By choosing different objective functions, their approach can be used for either the exploration of unknown environments or inspection of a given surface manifold in both a known and an unknown volume. 
Perhaps most closely related to our work is the  approach by Papachristos et al.~\cite{PapachristosMKD19} where uncertainty is considered during execution via a receding-horizon technique. Their work  starts by computing a path, then optimizes uncertainty, iterating to minimize localization and mapping uncertainty. However, as uncertainty optimization relies heavily on online updates (via tracked landmarks) of the robot's belief it cannot be used in an offline planning phase.


\section{Problem Definition}
\label{Sec:ProblemDefinition}
In this section, we provide a formal definition of the inspection-planning problem under execution uncertainty.
We start  by introducing basic definitions and notations, and discuss the uncertainity considerations (Sec.~\ref{subsec:notations}). Next, we formally describe the inspection-planning problem in the deterministic (uncertainity-free) regime (Sec.~\ref{subsec:inspection-planning}). 
Then, we  introduce the inspection-planning problem under execution uncertainty which will be the focus of this work (Sec.~\ref{subsec:Inspection-planning problem under execution uncertainty}).

\subsection{Basic definitions and notations}
\label{subsec:notations}
We have a holonomic robot~$\mathcal{R}$ operating in a workspace~$\mathcal{W} \subset~\mathbb{R}^3$ amidst a set of (known) obstacles~$ \mathcal{W}_{\rm obs} \subset \mathcal{W}$.\footnote{The assumption that the robot is holonomic is realistic as our motivating application is UAV inspection in which the UAV typically flies in low speed in order to accurately inspect the relevant region of interest. }
A \emph{configuration}~$q$ is a~$d$-dimensional vector uniquely describing the robot's pose (position and orientation) and let~$\mathcal{X}\subset\mathbb{R}^d$ denote the robot's \emph{configuration space}. \
Let~$\text{Shape}: \mathcal{X} \rightarrow 2 ^\mathcal{W}$ be a function mapping a configuration~$q$ to the workspace region occupied by~$\mathcal{R}$ when placed at~$q$ (here,~$2^\mathcal{W}$ is the power set of~$\mathcal{W}$).
We say that~$q \in \mathcal{X}$ is \emph{collision free} if~$\text{Shape}(q) \bigcap \mathcal{W}_{obs} = \emptyset$.

A path~$\pi = (q_0,q_1,\ldots) \subset \mathcal{X}$ is a sequence of configurations called  \emph{milestones} connected by straight-line edges~$\overline{q_iq_{i+1}}$.
A path~$\pi$ is \emph{collision free} if every configuration along the path (i.e., either one of the milestones or along the edges connecting milestones) is collision-free. We use the binary function~$\mathcal{C}(\pi)$ to express the \emph{collision state} of a path~$\pi$ where~$\mathcal{C}(\pi) = 0$ and~$\mathcal{C}(\pi)=1$ correspond to~$\pi$ being collision free and in-collision, respectively. Finally, we use~$\ell(\pi)$ to denote the length of a given path~$\pi$.

When a path $\pi$ is computed in an \emph{offline phase} to be later executed by the robot, we refer to it as a \emph{command path}.
Unfortunately, when following the command path, the system typically deviates from~$\pi$ due to different sources of uncertainty.
\check{In particular, we assume that the robot operates under two sources of uncertainty:~(i)~\emph{control uncertainty} and~(ii)~\emph{sensor uncertainty}. (A simple toy example demonstrating these concepts is detailed in Sec.~\ref{sec:method}.)}
The control uncertainty results from various sources, such as a mismatch between the robot model used during planning and the real robot model, and disturbances in the environment (e.g., wind gusts).
Sensor uncertainty (which is usually the main source of uncertainty in an urban environment) corresponds to navigation-model parameters that determine the robot's configuration which is not accurately known. Those parameters need to be modeled as random variables and may include biases of the inertial sensors or GNSS error terms~\cite{bookGroves}. \check{In order to counteract errors resulting from those two sources of uncertainty, a \emph{localization algorithm}, e.g., Kalman or particle filter, is invoked to compute the estimated robot location~\cite{thrun2005probabilistic}.}

To this end, we assume that we have access to distributions~$\Dloc$ from which the parameters governing the execution uncertainty~$\Nloc$ are drawn (see example in~Sec.~\ref{subsec:Running-example})\footnote{Having access to $\Dloc$ is a common assumption, see, e.g.,~\cite{gross2015robust,khaghani2016autonomous}.}. 
In addition, we assume to have access to the initial true location (i.e., there is no uncertainty in the initial configuration of any execution path regardless of the command path provided).
Finally, we assume that we have access to a black-box simulator, or motion model,~$\mathcal{M}$ that, given a command-path~$\pi$,
the uncertainty parameters~$\Nloc$ and the initial location, outputs the path that the robot will pursue starting from the initial location while following~$\pi$ under~$\Nloc$.  
Note that after~$\Nloc$ is drawn, the model is deterministic. \check{We mention that our approach is general and can be applied to any type of control and sensor uncertainty, as well as the accompanying estimation algorithm, so long that they can be faithfully simulated.}

\subsection{Inspection planning (without execution uncertainty)}
\label{subsec:inspection-planning}
In the inspection-planning problem, we receive as input a set of POI~$\mathcal{I} = \{\iota_1, \ldots, \iota_k \} \subset \mathcal{W}$ which should be inspected using some on-board inspection sensors (e.g., a camera). We model the inspection sensors as a mapping~$\mathcal{S}: \mathcal{X}\rightarrow 2^\mathcal{I}$ such that~$\mathcal{S}(q)$ denotes the subset of~$\mathcal{I}$ that are inspected from a configuration~$q\in \mathcal{X}$.
By a slight abuse of notation, we define~$\mathcal{S}(\pi):=\bigcup_{i} \mathcal{S}(q_i)$ to be the POI that can be inspected by traversing the path~$\pi=(q_0,q_1,\ldots)$. 
For simplicity, we only inspect POIs along milestones (rather than edges). We start with a simplified setting of the inspection problem which involves no uncertainty on the side of control or sensing. That is, a robot performing the inspection can precisely follow a path during execution time and its location is known exactly. Such a setting was  solved by~\iris~(see~Sec.~\ref{subsec:iris}). 

\vspace{3mm}
\begin{prob}[Deterministic problem]
\label{prob-1}
    In the \emph{inspection-planning problem} we wish to compute in an offline phase a path~$\pi$ that maximizes its coverage~$\vert {\mathcal{S}}(\pi) \vert$.
    Out of all such paths we wish to compute the paths whose length~${\ell}(\pi)$ is minimal.
\end{prob}
\vspace{3mm}

\textbf{Note.}~In practice, one may be interested in minimizing mission completion time or energy consumption (and not path length which is a first-order approximation for these metrics).
Optimizing for these metrics is slightly more complex and is left for future work.

Prob.~\ref{prob-1} can be defined for the 
\emph{continuous setting}~(i.e.,~when we consider all paths in~$\mathcal{X}$ between a given start configuration~$q_{\rm start}$ and goal configuration~$q_{\rm goal}$) or
for the \emph{discrete setting} where we restrict the set of available paths to those defined via a given \emph{roadmap}. Here, a roadmap~$\mathcal{G} = (\mathcal{V}, \mathcal{E})$ is a graph embedded in the configuration space~$\mathcal{X}$ such that each vertex is associated with a configuration and each edge with a local path connecting close-by configurations.
Roadmaps are commonly used in motion-planning algorithms (see, e.g.,~\cite{LaValle2006_Book,Salzman19})  and, as we will see, will be the focus of this work as well.

\subsection{Inspection-planning problem under execution uncertainty}
\label{subsec:Inspection-planning problem under execution uncertainty}

In the inspection-planning problem under execution uncertainty, we calculate a 
command-path~$\pi$ in an offline stage. To account for execution uncertainty, the following definitions extend the notion of path length, path coverage, and collision state to be their expected values: 
 \vspace{3mm}
\begin{dfn}
The \emph{expected} collision, coverage, and length probabilities are defined as:
\begin{subequations}
\begin{align}
\colProbBar
    &:= E_{\Nloc \sim \Dloc}
        \left[
            \mathcal{C}
                \left(
                    \mathcal{M}(\pi,\Nloc)
                \right)
              \right], \\
\vert \bar{\mathcal{S}}(\pi) \vert 
    &:=E_{\Nloc \sim \Dloc}
        \left[
            \vert
            \mathcal{S}
                \left(
                    \mathcal{M}(\pi,\Nloc)
                \right)
            \vert
        \right], \\
\bar{\ell}(\pi)
    &:= E_{\Nloc \sim \Dloc}
        \left[
            \ell
                \left(
                    \mathcal{M}(\pi,\Nloc)
                \right)
        \right],
        \end{align}
           \end{subequations}   
respectively. Here,~$\colProbBar \in \left[0,1 \right]$, 
where~$\colProbBar=0$ and~$\colProbBar=1$ correspond to~$\pi$ being guaranteed to be collision-free and in-collision, respectively.  
\end{dfn}

We are now ready to formally introduce the optimal inspection-planning problem under execution uncertainty.

\vspace{3mm}
\begin{prob}[Optimal problem]
\label{prob-2}
    In the \emph{optimal inspection-planning problem under execution uncertainty} 
    we are given a user-provided threshold~$\rho_\text{coll} \in [0, 1]$ and 
    we wish to compute in an offline phase a command-path~$\pi$ such that its expected execution collision probability is below~$\rho_\text{coll}$~(i.e.,~$\colProbBar\leq \rho_\text{coll}$), and which maximizes the expected coverage~$\vert \bar{\mathcal{S}}(\pi) \vert$. Of all such paths, we wish to choose the one whose expected length~$\bar{\ell}(\pi)$ is minimal.\footnote{\check{At first glance, one may be tempted to always provide a collision probability of zero. However, this comes at a computational cost---placing unnecessarily tight constraints may yield longer running times and, in extreme cases, problem infeasibility. See Sec.~\ref{sec:Theoretical guarantees} for more details.}} 
\end{prob}
\vspace{3mm}

Finally, we introduce a relaxation of the above problem to reduce its computational cost.
\vspace{3mm}
\begin{prob}[Sub-optimal problem]
\label{prob-3}
Let~$\pi^*$ be the solution to 
the optimal inspection-planning problem under uncertainty.
In addition, 
let~$\eps \geq 0$, and~$\kappa \in (0, 1]$ be user-provided approximation factors with respect to path length and coverage, respectively.
Then, in the \emph{sub-optimal inspection-planning problem under execution uncertainty} we wish to compute in an offline phase a command path~$\pi$ such that:
\begin{subequations}
\begin{align}
\colProbBar \leq & \rho_\text{coll},\\
\vert \bar{\mathcal{S}}(\pi) \vert \geq & \kappa \cdot \vert {\mathcal{S}}(\pi^*) \vert, \\
\bar{\ell}(\pi) \leq &  (1+\eps) \cdot {\ell}(\pi^*).
\end{align}
\end{subequations}
\end{prob}

Notice that, by setting~$\kappa=1$ and~$\eps=0$ the sub-optimal Prob.~\ref{prob-3} is equivalent to  Prob.~\ref{prob-2}. 
As we are only given a black-box model of~$\mathcal{M}$, it is infeasible to directly compute the expected values for coverage, collision probability, and length,~$\vert \bar{\mathcal{S}}(\pi) \vert$,~$\colProbBar$ and~$\bar{\ell}(\pi)$, respectively. 
As we will see, our approach will be to solve Prob.~\ref{prob-3} using estimates of these values.

\section{Algorithmic background}
\label{Sec:Algorithmic background}

In this section, we provide algorithmic background. We begin by describing  \iris~\cite{DBLP:journals/ijrr/FuKSA23}, a state-of-the-art algorithm for solving the continuous inspection-planning problem in the deterministic regime (Prob.~\ref{prob-1}). 
We then continue to outline the statistical methods we will use.
Throughout the text we assume familiarity with the A* algorithm~\cite{Hart1968_TSSC}.

\subsection{Incremental Random Inspection-roadmap Search (\iris)}
\label{subsec:iris}
\iris solves Prob.~\ref{prob-1} by incrementally constructing a sequence of increasingly dense graphs, or roadmaps, embedded in~$\X$ and computes an inspection plan over the roadmaps as they are constructed.
The roadmap~$\G = (\V,E)$ is a Rapidly-exploring Random Graph (RRG)~\cite{Karaman2011_IJRR} rooted at the start configuration (though other types of graphs, such as PRM*, can be used as well). 
For simplicity, when describing \iris below (and \irisuu later on), we focus on the behavior of the algorithm for a given roadmap. More information on how to construct such roadmaps can be found in~\cite{DBLP:journals/ijrr/FuKSA23}.

Let $\R_\G := \{\iota \in \R \vert \exists v \in \V \text{ s.t. }  \iota \in \S(v) \}$ be the set of all inspection points that can be inspected from some roadmap vertex.
To compute an inspection plan, \iris considers the \emph{inspection graph}~$\G_\S = (\V_\S, \E_\S)$  induced by the roadmap~$\G$.
Here, vertices are pairs comprised of a vertex~$u \in \V$ in the roadmap $\G$ and subsets of $\R_\G$. 
Namely, 
$\V_\S = \V \times 2^{\R_\G}$,
and note that~$\vert \V_\S \vert =  O\left(\vert \V \vert \cdot 2^{|\R_\G|} \right)$.
An edge~$e\in \E_\S$ between vertices~$(u, \R_u)$ and $(v, \R_v)$ exists if~$(u,v) \in \E$ and~$\R_u \cup \S(v) = \R_v$. The cost of such an edge  is simply the the length of the edge $(u,v)\in \E$, namely~$\ell(u,v)$.
The graph~$\G_\S$ has the property that a shortest path in the inspection graph corresponds to an optimal inspection path~$\pi^*$ over $\G$. 
However, the size of~$\G_\S$ is exponential in the number of POIs $\vert \I_\G  \vert$.
Thus, to reduce the runtime complexity \iris uses a search algorithm that approximates~$\pi^*$, which allows to prune the search space of paths over $\G_\S$.

{Specifically, the  approach for pruning the search space used by \iris is done through the notion of \emph{approximate dominance}, which allows to only consider paths that can significantly improve the quality (either in terms of length or the set of points inspected) of a given path. In particular, let~$\pi, \pi'$  be two paths in $\G$ that start and end at the same vertices and let $\varepsilon \geq 0$ and $\kappa \in (0,1]$ be some approximation parameters. We say that $\pi$ $(\eps,\kappa)$-dominates $\pi'$ if $\ell(\pi)\leq (1+\eps)\cdot \ell(\pi')$ and $|\S(\pi)|\geq \kappa\cdot |\S(\pi)\cup \S(\pi')|$. If $\pi$ indeed $(\eps,\kappa)$-dominates $\pi'$ then $\pi'$ can potentially be pruned. However, if we prune away approximate-dominated paths, we need to efficiently account for all paths that were pruned away in order to bound the quality of the solution obtained. This is done through the notion of \emph{potentially-achievable paths} described below.}

The search algorithm used by \iris employs an A*-like search over~$\G_\S$, where each node in the search tree is associated with a \emph{path pair} (\pp) corresponding to a vertex $v$ in $\G_\S$ (rather than only a path as in \astar). Here, a \pp is a tuple~$(\pi, \tilde{\pi})$, where~$\pi$ and~$\tilde{\pi}$ are the so-called \emph{achievable path} (\ap) and \emph{potentially achievable path} (\pap), respectively. The \ap~$\pi$ represents a realizable path in~$\G_\S$, from the start vertex to $v$, and is associated with two scalars corresponding to the path's length and coverage, respectively. The \pap~$\tilde{\pi}$ is a pair of scalars~$\tilde{\ell}, \tilde{S}$ representing length and coverage, respectively, which are used to bound the quality of any achievable paths to $v$ represented by a specific \pp. 
{Note that $\tilde{\pi}$ does not imply necessarily that there exists any path $\pi'$ from from the start to $v$ such that $\ell(\pi') = \ell(\tilde{\pi})$ and $\S(\pi) = \S(\tilde{\pi})$. It merely states that such a path \emph{could} exist.}

A \pp is said to be~$(\eps,\kappa)$-bounded if (i) the length of the \ap is no more than~$(1+\eps)$ times the length of the \pap and (ii)~the coverage of the \ap is at least~$\kappa$ percent of the coverage of the \pap. 

The search algorithm starts with a path pair rooted at the start vertex~$v_{\rm start}$ where both the \ap and the \pap  represent the trivial paths that only contain~$v_{\rm start}$ (i.e., the scalars associated with the length of the \ap and the \pap are zero and the scalars associated with the coverage of the \ap and the \pap is $\S(v_{\rm start})$). It operates in a manner similar to A*, with an OPEN list and a CLOSED set to track nodes that have not and have been considered, respectively.
Each iteration begins with popping a node from the OPEN list and checking if an inspection path has been found.
If this is not the case, the popped node from the OPEN list is inserted into the CLOSED set and the iteration continues.

The next step is \emph{extending} this node and testing whether its successors are \emph{dominated} by an existing node. If this is the case, the node is discarded. Otherwise, \iris tests whether the node can be \emph{subsumed} by or subsume another node.

These three core operations (extending, dominating, and subsuming) are key to the efficiency of \iris.
When extending a node, a path pair~$\pp_u = (\pi_u, \tilde{\pi}_u)$ to some vertex~$u$ is extended by an edge~$e = (u,v)\in \E_\S$ to create the path pair~$\pp_v = (\pi_v, \tilde{\pi}_v)$.
The length and coverage of the \ap of~$\pi_v$ are ~$\ell(\pi_u) + \ell(e)$ and~$\S(\pi_u) \cup \S(v)$, respectively. 
Similarly, the length and coverage of the \pap~$\tilde{\pi}_v$ are~$\ell(\tilde{\pi}_u) + \ell(e)$ and~$\S(\tilde{\pi}_u) \cup \S(\tilde{\pi}_v)$, respectively.
When testing  domination, two path pairs~$\pp_{u,1}$ and~$\pp_{u,2}$ to some vertex~$u\in \V_\S$ are considered.
$\pp_{u,1}$ is said to dominate~$\pp_{u,2}$ if both~$\ell(\pi_{u,1}) \leq \ell(\pi_{u,2})$ and~$\S(\pi_{u,2}) \subseteq \S(\pi_{u,1})$. In such a case,~$\pp_{u,1}$ is preferable to $\pp_{u,2}$ and so~$\pp_{u,2}$ can be discarded.
However, in many settings two path pairs will not dominate each other but their respective path lengths and coverage will be similar. To avoid maintaining and extending such similar path pairs, \iris uses the subsuming operation.
This operation is denoted by~$\pp_{u,1} \oplus \pp_{u,2}$, which creates a new path pair~$\pp_{u,3}$   whose \ap's length and coverage are identical to those of~$\pp_{u,1}$.
The \pap's length of~$\pp_{u,3}$ is the minimum \pap's length of~$\pp_1,\pp_2$, and the \pap's coverage is the union of the coverage of the \paps of~$\pp_1$ and $\pp_2$. 
Subsuming is only performed as long as the resultant \pp is~$(\eps,\kappa)$-bounded which allows to guarantee bounds on the solution quality.
For additional details, see~\cite{DBLP:journals/ijrr/FuKSA23}.

The algorithm's asymptotic convergence to an optimal solution is achieved through a process of iterative roadmap densification and parameter tightening. This iterative approach involves systematically refining the roadmap~$\G$ (i.e., adding vertices and edges) while progressively reducing the parameters~$\eps$ and~$\kappa$. 
%
Roadmap densification ensures that the algorithm considers larger sets of configurations, leading to more accurate and refined solutions. 
Meanwhile, the tightening of parameters focuses the algorithm on increasingly promising paths within the roadmap.

\subsection{Monte-Carlo methods \& confidence intervals}
\label{subsec:Monte-Carlo methods statistical tests}

Monte Carlo methods are a family of statistical techniques used to simulate and analyze complex systems that involve randomness. These methods involve generating multiple random samples and using them to estimate the value of a process being studied. Due to the finite number of samples, there is uncertainty regarding the true value of the process. To quantify this uncertainty, a common approach is CI and CL~\cite{hazra2017using,neyman1937outline}. 

A CI is a range of values likely to contain the true value of a population's parameter (such as its mean) with a certain confidence levels. For example, in our setting this could be the expected path length. The size of the CI reflects the uncertainty around the estimated value and is influenced by the number of samples used. As the number of samples increases, the accuracy of the estimate improves (i.e., CL increases) and the CI decreases. See additional background on the statistical tools we use in Appendix~\ref{app:stat}.


\section{Method}
\label{sec:method}
In this section, we present our method called \emph{IRIS under uncertainty}, or~\irisuu, to solve Prob.~\ref{prob-3}.
This is done by extending the algorithmic framework of \iris to consider execution uncertainty within the inspection-planning algorithm.
We start  with a general description of our algorithmic approach (Sec.~\ref{subsec:alg-approach}), and then 
describe how the operations used in \iris are modified to account for localization uncertainty (Sec.~\ref{subsec:Modifying primitive search operations}).
This is followed by a toy scenario that is used to demonstrate the key newly-introduced definitions and operations (Sec~\ref{subsec:Running-example})
We conclude  by describing  how those updated operations are used by \irisuu to compute an inspection path (Sec.~\ref{subsec:Algorithm flow}).

\subsection{\irisuu---Algorithmic approach}
\label{subsec:alg-approach}

A na\"ive approach to address execution uncertainty is to penalize paths with high localization uncertainty.
We describe one such approach as a baseline in Sec.~\ref{subsec:The compared methods}.
As we will see, while highly efficient in collision avoidance, even minor deviations from the command path due to execution uncertainty can lead to discrepancies between the intended POI coverage of the command path and the actual path taken during execution, particularly when obstacles are present. Thus, instead of reasoning about localization uncertainty, in \irisuu, we directly consider and maximize POI coverage.

Unfortunately, we cannot directly compute the expected values for coverage, collision probability, and length,~$\vert \bar{\mathcal{S}}(\pi) \vert$,~$\colProbBar$ and~$\bar{\ell}(\pi)$, respectively.
Thus, we tackle Prob.~\ref{prob-3} using the estimated values~$\estS, \colProbHat$ and~$\estL$ instead of the expected values~$\vert \bar{\mathcal{S}}(\pi) \vert,\colProbBar$ and~$\bar{\ell}(\pi)$ such that the command path~$\pi$, satisfies:
\begin{subequations}
\begin{align}\label{eq:estimate_prob_3}
\colProbHat \leq & \rho_\text{coll},\\
\vert \bar{\mathcal{S}}(\pi) \vert \geq & \kappa \cdot \vert \hat{\mathcal{S}}(\pi^*) \vert, \\
\bar{\ell}(\pi) \leq & (1+\varepsilon) \cdot \hat{\ell}(\pi^*).
\end{align}
\end{subequations}
Here, 
the estimated coverage~$\estS$, 
the estimated collision probability~$\colProbHat$ 
and the estimated  path length~$\estL$ are computed by simulating $m\geq 1$
different executions (with respect to uncertainty) of $\pi$ (see details below).

Similar to \iris, \irisuu solves the inspection-planning problem   by sampling an initial roadmap.
It then iteratively 
(i)~plans a command path on this roadmap and 
(ii)~densifies the roadmap and refines the algorithm's parameters.
Importantly, the focus of this work is on the path-planning part of the algorithmic framework wherein a command path is computed for a given roadmap.
This is visualized in Fig.~\ref{fig:AlgorithmicFlow} where the path-planning part is highlighted.
For completeness, we reiterate that graph refinement is done by continuing to grow the RRG as described in Sec.~\ref{Sec:Algorithmic background}.

%


\begin{figure}[tb]
    \centering
     \includegraphics[width=0.5\textwidth]{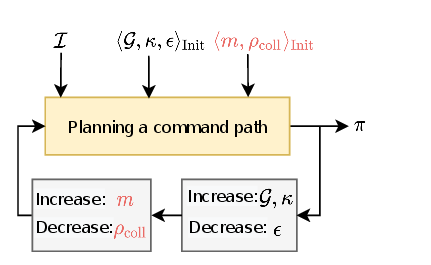}
     \caption{Planning a command path~$\pi$ to inspect set of POI~$\mathcal{I}$ given the input of \iris---~$\G$,~$\kappa$ and~$\varepsilon$, and the additional parameters of \irisuu---~$m$ the~$\rho_{\text{coll}}$.
    }
\label{fig:AlgorithmicFlow}
\end{figure}

Specifically, we start by initializing  our algorithm by sampling~$m$ different parameters~$\Nloc^1, \ldots, \Nloc^m$ from~$\Dloc$.
%
For any command path~$\pi$ considered by the algorithm, we will estimate~$\expL$ and~$\expS$ by simulating~$m$ execution paths~$\pi^e_1, \ldots, \pi^e_m$ using the~$m$ motion models. Namely, 
\begin{equation}\label{eq:motion_model_paths}
    \pi^e_i:=\mathcal{M}(\pi, \Nloc^i).
\end{equation}
%

To estimate the expected coverage, let~$\xi_{i,j} \in \{0,1\}$ be a variable that will be set to one if path~$\pi^e_i$ covers the~$j$'th POI.
Namely,
\begin{equation}\label{eq:calculatePOI-vertex} 
\xi_{i,j}:=
    \begin{cases}
     1 & \iota_j \in \mathcal{S}(\pi^e_i), \\
     0 & \text{else}.
    \end{cases}
\end{equation}
Then, we define for the command path~$\pi$ the \emph{inspection probability vector}~(IPV):
\begin{equation}
\label{eq:IPV}
    \textup{IPV}(\pi) := \{ \hat{p}_1^\pi, \ldots, \hat{p}_k^\pi \}.
\end{equation}
where:
\begin{equation}
    \hat{p}_j^\pi: = \frac{1}{m} \sum_{i=1}^{i=m} \xi_{i,j},
\end{equation}
is the estimated probability that the~$j$'th POI is viewed when executing the command path~$\pi$.
Finally, the estimated expected coverage is defined as:
\begin{equation}\label{eq:estimated expected coverage}
    \vert \estS \vert 
    := 
    \sum_{j=1}^{j=k} \hat{p}^\pi_{j}.
\end{equation}

\begin{assumption}
\label{ass:iid}
Here, we assume that the probability~$\hat{p}^\pi_{j}$ of inspecting each POI $j$ is independent of other POIs. 
\end{assumption}
As we will see, 
Assumption~\ref{ass:iid} will both
(i)~simplify the analysis and 
(ii)~will not hinder the guarantees obtained from the analysis in practice.
Relaxing the assumption is left for future work.

To estimate the collision probability~$\colProbHat$ of a path~$\pi$,~\irisuu maintains for the command path~$\pi$ a \emph{collision vector}~(CV):
\begin{equation}
\label{eq:CV}
    \rm{CV}(\pi) := \{ \zeta_1^\pi,\zeta_2^\pi, \ldots  \}.
\end{equation}
Here,~$\zeta_j^\pi$ indicates  whether the path~$\pi_j^e$ was found to be in a collision.
Namely:
\begin{equation}
\zeta_{j}^\pi:=
    \begin{cases}
     1 & \pi_j^e~ \text{is in collision}, \\
     0 & \text{else}.
    \end{cases}
\end{equation}

Subsequently, the estimated collision probability~$\colProbHat$ of a path~$\pi$ is defined as:
\begin{equation}\label{eq:definition_of_p_coll}
    \colProbHat: = \frac{1}{m} \sum_{j=1}^{j=m} \zeta_j^\pi.
\end{equation}

Similarly, the estimated expected path length~$\expL$ is defined as:
\begin{equation}\label{eq:estimated expected length}
\estL:=
    \frac{1}{m} \sum_{j=1}^{j=m} \ell~(\pi^e_j).
\end{equation}

\subsection{\irisuu---Modified search operations}
\label{subsec:Modifying primitive search operations}
Recall that, while \iris maintains the POIs inspected using a  set representation, \irisuu maintains an inspection probability vector (IPV) for each path.
In addition, \irisuu maintains estimations of the expected path length and expected  collision probability of its command path.   
This requires to modify  node operations used by \iris's search algorithm to account for the uncertainty values. 
Next we will explain how to modify the operations used in \irisuu's \astar-like search that were originally used in \iris. 
We start by formally defininig nodes in Sec.~\ref{subsec:Node definition} and detail the node extension, collision, and domination operations in Sec.~\ref{subsec:Node extension},~\ref{subsec:Node collision} and~\ref{subsec:Node domination}, respectively. Finally, we describe the node subsuming operation in Sec.~\ref{subsec:Node subsuming} and termination criteria in Sec.~\ref{subsec:Termination criteria}.

\vspace{2mm}
\subsubsection{Node definition}
\label{subsec:Node definition}
A node~$n$ in our search algorithm is a tuple:
\begin{equation}
n = 
\langle
    u, \pi_u, \Pi^e_u, \hat{\rm{IPV}}_u, \hat{\ell}_u,
    \tilde{\rm{IPV}}_u,\tilde{\ell}_u,\hat{\mathcal{C}}_u
\rangle,
\end{equation}
where $u\in \V$ is a roadmap vertex,
$\pi_u$ is a path from the start vertex~$v_{\rm start}\in \G$ to~$u$,
$\Pi_u^e = \{ \pi_{1,u}^e, \ldots, \pi_{m,u}^e\}$ are~$m$ simulated execution paths calculated using Eq.~\eqref{eq:motion_model_paths}, 
$\hat{\rm{IPV}}_u = \{ \hat{p}_1^{\pi_u}, \ldots, \hat{p}_k^{\pi_u} \}$ is the estimated IPV (see Eq.~\eqref{eq:IPV}),
and~$\hat{\ell}_u$ is the estimated path length (see Eq.~\eqref{eq:estimated expected length}) with respect to the command path $\pi_u$.
In addition,~$\tilde{\textup{IPV}}_u = \{ \hat{p}_1^{\tilde{\pi}_u}, \ldots, \hat{p}_k^{\tilde{\pi}_u} \}$ and~$\tilde{\ell}_u$ are the estimated IPV and estimated length of the \pap.
Finally,~$\hat{\mathcal{C}}_u$ is the estimated collision probability of~$\pi_u$~(see Eq.~\eqref{eq:definition_of_p_coll}).

We define the \emph{initial node} as:
\begin{equation}
n_\text{init}:=
    \langle
    v_{\rm start}, \{ v_{\rm start} \}, \Pi^e_{v_{\rm start}}, 
    \text{IPV}_{v_{\rm start}}, 0,
    \text{IPV}_{v_{\rm start}},0,0
\rangle,
\end{equation}
where the command path~$\pi$ of~$n_\text{init}$ consists of the trivial path starting and ending at~$v_{\rm start}$ whose estimated length and collision are initialized to zero (we assume that~$v_{\rm start}$ corresponds to a collision-free configuration).
In addition, the~$m$ MC starting configurations are set to be~$v_{\rm start}$ (namely,~$\Pi^e_{v_{\rm start}} = \{ v_{\rm start}, \ldots, v_{\rm start}\}$)
which induce the initial inspection probability vector~$\text{IPV}_{v_{\rm start}} = \{ \S(v_{\rm start}), \ldots, \S(v_{\rm start}) \}$. 
Finally, analogously to \iris, the IPV and estimated length of the \pap are initialized to be the same as the \ap.

\vspace{2mm}
\subsubsection{Node extension}
\label{subsec:Node extension}
Let~$u,v \in \mathcal{V}$ be two roadmap vertices such that~$(u,v) \in \mathcal{E}$.
and let~$n_u= 
\langle
    u, \pi_u, \Pi^e_u, \hat{\rm{IPV}}_u, \hat{\ell}_u, \tilde{\rm{IPV}}_u,\tilde{\ell}_u,\hat{\mathcal{C}}_u
\rangle$ be a node in the search algorithm associated with vertex~$u$.
We define the operation of \emph{extending}~$n_u$ by the edge~$(u,v)$ as creating a new node 
$n_v= 
\langle
    v, \pi_v, \Pi^e_v, \hat{\rm{IPV}}_v, \hat{\ell}_v, \tilde{\rm{IPV}}_v,\tilde{\ell}_v,\hat{\mathcal{C}}_v
\rangle$
such that:
\begin{itemize}
    \item $\pi_v$ is the result of concatenating~$\pi_u$ with the path~$\pi_{u\rightarrow v}$ from~$u$ to~$v$.
    Namely,
       \begin{equation}
                  \pi_v:= \pi_u \circ \pi_{u\rightarrow v}.
        \end{equation}

    \item $\Pi^e_v := \{\pi_{1,v}^e \ldots \pi_{m,v}^e \}$ is a set of execution paths such that~$\pi^e_{j,v}:=\mathcal{M}(\pi_v, \Nloc^j)$. Notice that this can be efficiently computed by denoting: 
  \begin{equation}
        \pi^e_{j,u\rightarrow v}:=\mathcal{M}(\pi_{u\rightarrow v}, \Nloc^j, \pi^e_{j,u}),
    \end{equation}
    and setting: 
    \begin{equation}
        \pi^e_{j,v}:= \pi^e_{j,u} \circ \pi^e_{j,u\rightarrow v}.
    \end{equation}
    \item $\hat{\rm{IPV}}_v: = \{ \hat{p}_1^{\pi_v}, \ldots, \hat{p}_k^{\pi_v} \}$ is obtained by leveraging the assumptions, that (i)~POI inspections are conducted at vertices (see Sec.~\ref{subsec:notations}) and that~(ii)~each inspection is independent~(see Assumption.~\ref{ass:iid}).
   \begin{equation}\label{eq:extend_IPV_ap}
    \forall j\in [1,k]~\hat{p}_j^{\pi_v} := 
     1-(1-\hat{p}^{\pi_{u \rightarrow v}}_{j}) \cdot~(1-\hat{p}_j^{\pi_{u}}).     
   \end{equation}
   
    \item $\hat{\ell}_v$ is the estimated path length of~$\pi_v$. Notice that this can be efficiently computed by setting:
  \begin{equation}
              \hat{\ell}_v: = \hat{\ell}_u + \hat{\ell}(\pi_{u\rightarrow v}).
    \end{equation}
    
    \item The estimated IPV of the \pap is updated such that: 
  \begin{equation}
            \forall j\in [1,k]~\tilde{p}_j^{\pi_v} := 1-(1-\hat{p}^{\pi_{u \rightarrow v}}_{j}) \cdot~(1-\hat{p}_j^{\tilde{\pi}_{u}}).
    \end{equation}
      and the estimated path length of the \pap is:
  \begin{equation}
            \tilde{\ell}_v: = \tilde{\ell}_u + \hat{\ell}(\pi_{u\rightarrow v}).
    \end{equation}

    \item $\hat{\mathcal{C}}_v$ is the collision probability of node~$v$, such that:
  \begin{equation}
          \hat{\mathcal{C}}_v :=
    1-\left(1-\hat{\mathcal{C}}_u \right) \cdot \left(1-\hat{\mathcal{C}}_{u \rightarrow v} \right).
    \end{equation}

\end{itemize}

\vspace{2mm}
\subsubsection{Node collision}
\label{subsec:Node collision}
Recall that the collision probability of a node estimates the probability that the command path associated with~$n$ will intersect an obstacle.
Now, let~$n_u= 
\langle
    u, \pi_u, \Pi^e_u, \hat{\rm{IPV}}_u, \hat{\ell}_u, \tilde{\rm{IPV}}_u,\tilde{\ell}_u,\hat{\mathcal{C}}_u
\rangle$ be a node in the search algorithm associated with vertex~$u$.
Then, given a user-defined threshold~$\rho_\text{coll} \in [0,1]$ a node~$n$ will be considered \emph{in collision}~(and hence pruned by the search) if its collision probability~$\hat{\mathcal{C}}_u$ satisfies $\hat{\mathcal{C}}_u \geq \rho_\text{coll}$.

\vspace{2mm}
\subsubsection{Node domination}
\label{subsec:Node domination}
As in many \astar-like algorithms, node domination is used to prune away nodes that cannot improve the solution compared to other nodes expanded by the algorithm.
We introduce a similar notion that accounts for both converge~(via IPVs) and path length.

Specifically, let
$n_1$ and~$n_2$ be two nodes that start at $v_{\text{start}}\in \V$ and end at the same vertex~$u\in \V$.
Here, we assume that for each~$i \in \{1,2\}$,
\begin{equation}
    n_i= 
\langle
    u, \pi_{u,i}, \Pi^e_{u,i}, \hat{\rm{IPV}}_{u,i}, \hat{\ell}_{u,i}, \tilde{\rm{IPV}}_{u,i},\tilde{\ell}_{u,i},\hat{\mathcal{C}}_u
\rangle.
\end{equation}
\check{Then, we say that~$n_1$ \emph{dominates}~$n_2$ if $n_1$'s PAP is strictly better than $n_2$'s. Namely, if}
\begin{equation}
    \forall j \in [1 \ldots k],~
    \hat{p}_j^{\tilde{\pi}_1} \geq \hat{p}_j^{\tilde{\pi}_2}  
\quad
\text{and} 
\quad
    \tilde{\ell}(\pi_1) \leq \tilde{\pi}(\pi_2). 
    \end{equation}

\vspace{2mm}
\subsubsection{\ek-bounded nodes \& node subsuming  }
\label{subsec:Node subsuming}
Similar to \iris, we need to ensure that the \pap bounds the \ap given the user-provided parameters~$\eps >0$ and~$\kappa \in [0,1]$.
Specifically, let 
$
    n_u = 
        \langle
            u, \pi_u, \Pi^e_u, \hat{\rm{IPV}}_u, \hat{\ell}_u,
            \tilde{\rm{IPV}}_u,\tilde{\ell}_u,\hat{\mathcal{C}}_u
        \rangle
$
be a node such that 
$\hat{\rm{IPV}}(\pi_{u}) := \{ \hat{p}_1^{\pi_{u}}, \ldots, \hat{p}_k^{\pi_{u}} \}$
and
$\tilde{\rm{IPV}}(\pi_{u}) := \{ \tilde{p}_1^{\pi_{u}}, \ldots, \tilde{p}_k^{\pi_{u}} \}$
are the IPVs of~$n$'s \ap and \pap, respectively.
Similarly, let~$\hat{\ell}_u$ and~$\tilde{\ell}_u$ be the estimated lengths of~$n$'s \ap and \pap, respectively.
We say that~$n$ is \emph{\ek-bounded} if:
\begin{equation}\label{eq:ek-bounded}
    \sum_{j=1}^{j=k}  \hat{p}_j^{\pi_u} 
    \geq 
\kappa  \cdot
\sum_{j=1}^{j=k}  \tilde{p}_j^{{\pi}_u} 
\quad
    \text{and} 
\quad
\hat{\ell}_u \leq~(1+\varepsilon) \cdot \tilde{\ell}_u.
\end{equation}

Similar to \iris, \ek-bounded nodes  will be used together with node subsuming to reduce the number of paths considered by the search while retaining bounds on path quality.
Specifically we define node subsuming  as follows:
Let~$n_1$ and~$n_2$ be two nodes both starting at the same vertex and ending at the same vertex~$u$ 
such that:
\begin{equation}
    n_i= 
\langle
    u, \pi_{u,i}, \Pi^e_{u,i}, \hat{\rm{IPV}}_{u,i}, \hat{\ell}_{u,i}, \tilde{\rm{IPV}}_{u,i},\tilde{\ell}_{u,i},\hat{\mathcal{C}}_{u,i}
\rangle.
\end{equation}
Then, the operation of~$n_1$ \emph{subsuming}~$n_2$ 
(denoted as~$n_1 \oplus n_2$) 
will create the new node (that will be used to replace $n_1$ and prune $n_2$): 
\begin{equation}
    \begin{split}    
 n_1 \oplus n_2: &= \\ 
 n_3 &=  
        \langle
            u, \pi_{u,3}, \Pi^e_{u,3}, \hat{\rm{IPV}}_{u,3}, \hat{\ell}_{u,3},
            \tilde{\rm{IPV}}_{u,3},\tilde{\ell}_{u,3},\hat{\mathcal{C}}_{u,3}
        \rangle.
    \end{split}
\end{equation}

Here, the components  of~$n_3$'s \ap are identical to~$n_1$'s. 
Namely,~$n_3$'s components associated with the \ap are defined as follows:
    \begin{align*}
        \pi_{u,3} &= \pi_{u,1},\\  
        \Pi^e_{u,3} &= \Pi^e_{u,1},\\
        \hat{\rm{IPV}}_{u,3} &= \hat{\rm{IPV}}_{u,1},\\
        \hat{\ell}_{u,3}&= \hat{\ell}_{u,1},\\
        \hat{\mathcal{C}}_{u,3} &= \hat{\mathcal{C}}_{u,1},
    \end{align*}
and~$n_3$'s components associated with the \pap are defined as follows:
\begin{equation*}
    \begin{split}
      \tilde{\ell}_{u,3} &:= \min{(\tilde{\ell}_{u,1},\tilde{\ell}_{u,2})},\\
      \tilde{p}_j^{\pi_{u,3}} &:= \max{(\tilde{p}_j^{\pi_{u,1}},\tilde{p}_j^{\pi_{u,2}})}.\\
      ~
    \end{split}
\end{equation*}

\subsubsection{Termination criteria}
\label{subsec:Termination criteria}
To terminate the search in \irisuu given a node~$n_v$ we check whether~$n_v$ satisfies:
\begin{equation}
\label{eq:goal_check}
    \sum_{j=1}^{j=k} \hat{p}_j^{\hat{\pi}_u}  \geq  k \cdot \kappa.
\end{equation}
Here~$\{ \hat{p}_1^{\tilde{\pi}_v}, \ldots, \hat{p}_k^{\tilde{\pi}_v} \}$ is the IPV of~$n$'s \ap and~$k$ is the number of POIs.



\vspace{2mm}
\subsection{Modified search operation---illustrative example}
\label{subsec:Running-example}
\begin{figure}[tb]
    \centering
     \includegraphics[width=0.5\textwidth]{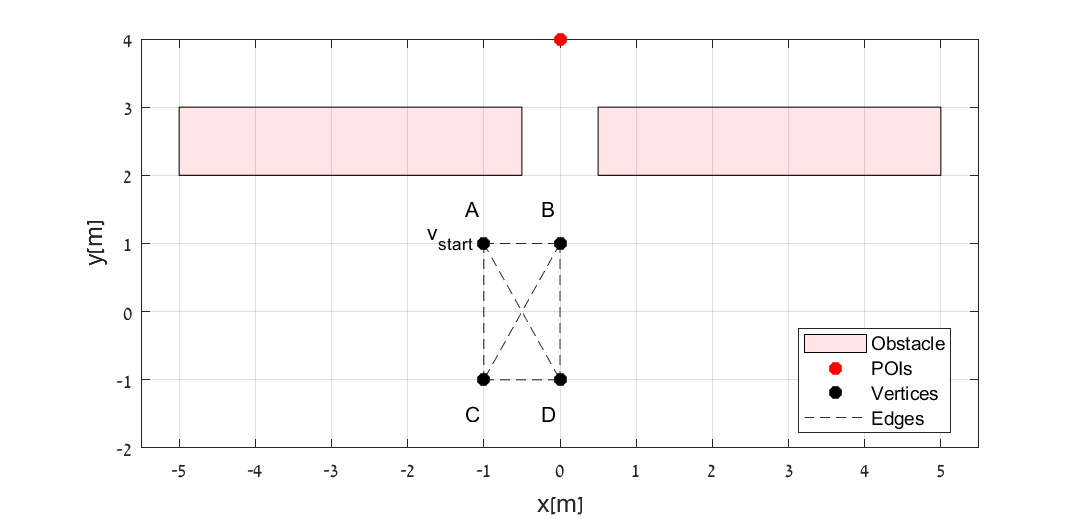}
     \caption{Toy scenario for running example with one POI (red dot), two obstacles (red rectangles), and four vertices (black dots) connected with six edges (black dashed lines).
    }
\label{fig:RunningExample}
\end{figure}

Consider the toy problem illustrated in Fig.~\ref{fig:RunningExample}. 
The roadmap~$\G$ contains four vertices~$A$,~$B$,~$C$, and~$D$ which represent configurations of a point robot (namely, each configuration defines the~$(x,y)$ location of the point robot).
Here, the single POI can be seen from any configuration as long as the straight-line connecting them does not intersect an obstacle.

We use a simple toy motion model~$\Mexample$ to define the execution localization uncertainty in which we assume that the position of every configuration along the command path is normally distributed around the position of the corresponding configuration.
Specifically, parameters $\Nloc:=(r,\theta)$ are drawn from the distribution $\Dloc$ such that~$r \sim |\mathcal{N}(0,1)|$, and~$\theta \sim \mathcal{N}(0,2\pi)$.
Now, given a command path~$\pi^c$ with $n>1$ configurations~$
\pi^c = \{ \langle x_1^c,y_1^c \rangle, \ldots \langle x_n^c,y_n^c \rangle \}
$,
the corresponding executed path (which is also a random variable) is~$\pi_e = 
\{ \langle x_1^e,y_1^e \rangle, \ldots \langle x_n^e,y_n^e \rangle \}
$
such that~$x_i^e = x_i^c + r \cdot \cos \theta$ and~$y_i^e = y_i^c + r \cdot \sin \theta$ for $i>1$. 

Note that 
(i)~as we assume that there is no uncertainty in the initial location, $\langle x_1^c,y_1^c \rangle = \langle x_1^e,y_1^e \rangle$
and that
(ii)~in contrast to the general setting, here uncertainty is only a function of the current configuration and \emph{not} of the entire command path. 

Finally, in the running example, we assume that the algorithm uses three MC planning samples~$(m = 3)$ and the  start configuration is located  at vertex~$A$.

\vspace{2mm}
\subsubsection{Node definition---example}
For the running example in our toy problem,  the first location of each MC planning sample is  at~$v_{\rm start}$ and the POI cannot be inspected from that location. 
Specifically,~$n_\text{init}$ is defined such that:
\begin{align*}
    v_{\rm start} =& A,\\
    \Pi^e_{v_{\rm start}} =& \{ \langle -1,1 \rangle; \langle -1,1 \rangle;\langle -1,1 \rangle \},\\ 
    \text{IPV}_{v_{\rm start}} =& \{ \hat{p}_1^{\pi_{v_{\rm start}}}\} = \{ 0\}.
\end{align*}

\vspace{2mm}
\subsubsection{Node extension---example}
Extending the node~$n_\text{init}$  by  edge~$e(A,B)$ will result in a new node:
$$
n_B = \langle
B, \pi_B, \Pi^e_B, \hat{\rm{IPV}}_B, \hat{\ell}_B, \tilde{\rm{IPV}}_B,\tilde{\ell}_B,\hat{\mathcal{C}}_B
\rangle,$$
where~$\pi_B = \{A,B\}$ and~$\Pi^e_B$ is calculated using~$\Mexample$ to be: 
\begin{align*}
\Pi^e_B = \{ & \langle -1,1 \rangle,\langle 0.6,2.1 \rangle;\\
&\langle -1,1 \rangle, \langle 0.5,1.2 \rangle;\\
&\langle -1,1 \rangle, \langle -0.2,1.6 \rangle \}.
\end{align*}

As a result, since POI~$1$ can only be seen from~$\Pi^e_{B,2}$ and~$\Pi^e_{B,3}$, then~$\hat{p}^{\pi_{A \rightarrow B}}_{1} \cong 0.67$.
Using these values, we can calculate~$\hat{p}_1^{\pi_B}$ and~$\hat{p}_1^{\tilde{\pi}_B}$ as:
\begin{align*}
   \hat{p}_1^{\pi_B} =& 1 - (1 - \hat{p}^{\pi_{A \rightarrow B}}_{1}) \cdot (1 - \hat{p}_1^{\pi_{A}}) =\\
   =& 1 - (1-0.67) \cdot (1-0) = 0.67,\\
   \hat{p}_1^{\tilde{\pi}_B} =& \max(\hat{p}^{\pi_{A \rightarrow B}}_{1},\hat{p}_1^{\tilde{\pi}_{A}}) =0.67.
\end{align*}
Thus,~$\hat{\rm{IPV}}_B = \tilde{\rm{IPV}}_B = 0.67$.
As~$\hat{\ell}(\pi_{u\rightarrow v}) \cong 1.48$ (using~$\Pi^e_B$), we have that~$\hat{\ell}_B = \tilde{\ell}_B = 0+1.48$.
Finally,~$\Pi^e_{B,1}$ collides  with the middle obstacle, and~$\hat{\mathcal{C}}_B = \frac{1}{3}$.

\vspace{2mm}
\subsubsection{Node collision---example}
Assuming~$\rho_\text{coll} =0$ (i.e., we only allow collision-free paths), then~$\hat{\mathcal{C}}_B = \frac{1}{3} \geq \rho_\text{coll}$ is considered in  collision and this extension is discarded.

\vspace{2mm}
\subsubsection{Node domination---example}
Assume that node~$n_{{D},1}$ represents the command path~$A\mbox{-}D $ in the running example and its IPV and estimated length of the \ap and the \pap are:
  \begin{align*}
     \hat{\rm{IPV}}_{{D},1} = \{0.67\}, \quad   &\hat{\ell}_{{D},1} = 2.4  \\
     \tilde{\rm{IPV}}_{{D},1} = \{ 0.67\}, \quad & \tilde{\ell}_{{D},1} = 2.4.
  \end{align*}
  In addition, assume that later in the search node~$n_{{D},2}$ represents another command path~$A\mbox{-}C\mbox{-}D$. 
  This path  also reached vertex~$D$ and its IPV and estimated length of the \ap and the \pap are:
  \begin{align*}
      \hat{\rm{IPV}}_{{D},2} = \{ 0.77 \}, \quad &\hat{\ell}_{{D},2} = 2.9  \\
     \tilde{\rm{IPV}}_{{D},2} = \{ 0.77\}, \quad & \tilde{\ell}_{{D},2} = 2.9.
  \end{align*}
  Here, the estimated length of~$2.9$ is the sum of the estimated lengths of~$A\mbox{-}C$ and~$C\mbox{-}D$. 
  In addition, the IPV of the \ap and \pap contains a probability of~$0.33$ to inspect the POI from vertex~$C$ and~$0.67$ to inspect it from vertex~$D$.
  Thus, the value of~$\hat{\rm{IPV}}_{{D},2}$ equals~$1-(1-0.67)\cdot (1-0.33) \cong 0.77$ (using Eq.~\eqref{eq:extend_IPV_ap}). 
  Similarly, the value of~$\tilde{\rm{IPV}}_{{D},2}$ is also equal to~$0.77$.
  Notice that, the path lengths and IPVs of the \ap and \pap are identical here. This will change shortly when we introduce node subsuming (Sec~\ref{subsec:Node subsuming}).
  
  Here, despite that 
  $\tilde{\ell}_{{D},1} = 2.4\leq  2.9 = \tilde{\ell}_{{D},2}$ 
  (namely, the path to~$n_{D,1}$ is shorter than the path to $n_{D,2}$),
  we have that 
  $\hat{p}_1^{\tilde{\pi}_{{D},1}} =  0.67 \leq  0.77 = \hat{p}_1^{\tilde{\pi}_{{D},2}}$
  (namely, the path to~$n_{D,1}$ has a smaller probability of inspecting the POI).
  Thus,~$n_{{D},1}$ does not dominates~$n_{{D},2}$.

\vspace{2mm}
\subsubsection{Node subsuming---example}
Consider the operation~$n_{{D},1} \oplus n_{{D},2}$.
This results in  a new node~$n_{{D},3}$ such that its IPV and estimated length of the \ap and the \pap are:
    \begin{align*}
     \hat{\rm{IPV}}_{{D},3} &= \{ 0.67\},\\
     \tilde{\rm{IPV}}_{{D},3} &= \max ( \{ 0.67\}, \{ 0.77 \}) =  \{0.77\} \\ 
     \hat{\ell}_{{D},3} &= 2.4,\\
     \tilde{\ell}_{{D},3}& = \min (2.4,2.9) = 2.4.
  \end{align*}
  Here, if we choose~$\kappa = 0.85$ and $\varepsilon = 0$, then~$n_{{D},3}$ is \ek-bounded, namely:
  \begin{align*}
    \sum_{j=1}^{i=k}  \hat{p}_j^{\pi_{{D},3}} = 0.67;
    &
    \quad
    \kappa  \cdot \sum_{j=1}^{i=k}  \hat{p}_j^{\tilde{\pi}_{{D},3}}  = 0.85 \cdot 0.77 \cong 0.66;\\
    \hat{\ell}_{{D},3} = 2.4;
    &
    \quad
    (1+\varepsilon) \cdot \tilde{\ell}_{{D},3} = 2.4.
  \end{align*}
and indeed both $0.67 \geq 0.66$ and $2.4 \geq 2.4$

\vspace{2mm}
\subsubsection{Termination citeria---example}
As a result of the subsume operation~$n_{{D},3} = n_{{D},2} \oplus n_{{D},1}$ we have that the IPV of~$n_{{D},3}$'s  \ap and \pap are:
    \begin{align*}
     \hat{\rm{IPV}}_{{D},3} = \tilde{\rm{IPV}}_{{D},3} =  \{ 0.77\}. 
  \end{align*}
Now, assume that~$\kappa = 0.97$ and recall that we have one POI.
    As $0.77 < 1 \cdot 0.97$, the algorithm cannot terminate.
    However, assume we extend this path by returning to vertices~$C$ and~$D$ to obtain the command path~$A\mbox{-}C\mbox{-}D\mbox{-}C\mbox{-}D\mbox{-}C\mbox{-}D$ in which we perform inspection three times from~$C$ and and three times from~$D$ and   recall that the coverage probability of~$C$ and~$D$ equals~$0.33$ and~$0.67$, respectively.
    Thus, the coverage of the command path~$A\mbox{-}C\mbox{-}D\mbox{-}C\mbox{-}D\mbox{-}C\mbox{-}D$ equals:
    $$1-\underbrace{(1-0.33)^3}_{\text{Three times~$C$}}\cdot\underbrace{(1-0.67)^3}_{\text{Three times~$D$}} \cong 0.98.$$
    Finally,
    $$
    \sum_{j=1}^{j=k} \hat{p}_j^{\hat{\pi}_{A-C-D-C-D-C-D}} = 0.98 \geq  1 \cdot 0.97,
    $$ 
    and the algorithm terminates.

\vspace{2mm}
\subsection{\irisuu---Algorithmic description}
\label{subsec:Algorithm flow}

\begin{algorithm}
    \caption{\irisuu}
    \label{alg:irisuu}
    \hspace*{\algorithmicindent} \textbf{Input: } 
       $\langle \mathcal{G},m,\kappa,\varepsilon,\rho_{\text{coll}} \rangle$
        \\
    \hspace*{\algorithmicindent} \textbf{Output: } 
       \text{Command path } $\pi$
    \begin{algorithmic}[1]
    \State Initialize~$n_{\text{init}}$
            \hfill \texttt{//See~\ref{subsec:Node definition}}
    \State OPEN~$\gets~n_{\text{init}}$,
            \hspace{2mm}
           CLOSED~$\gets~\emptyset$
    \While {OPEN~$\ne~\emptyset$} 
    
        \State~$n\gets$ OPEN.extract\_node\_with\_max\_coverage() \label{line:check_goal_init}
        \State CLOSED.insert($n$)
        \If {$n$.is\_goal\_node()}
            \hfill \texttt{//See~\ref{subsec:Termination criteria}}
            \State 
                \Return~$ n.\pi~$
                \hfill \texttt{//command path }
        \EndIf \label{line:check_goal_end}

        \vspace{1mm}
    \For{$v \in \text{neighbour}(u,\G)$} \label{line:extending_init}
        \hfill \texttt{//$u$ is~$n$'s vertex}
        {\State~$n' \gets$~$n$.extend($u,v$) 
            \hfill \texttt{//See~\ref{subsec:Node extension}} \label{line:extending_end}
        \If{$n$.is\_in\_collision() }
            {\hfill \texttt{//See~\ref{subsec:Node collision}}\label{line:collision_init}}
        \State \textbf{continue}
        \EndIf \label{line:collision_end}
        
        \vspace{1mm}
        \State valid = \textbf{True}
        \For{$n'' \in$ CLOSED with vertex~$v$}  \label{line:dominates_init}
            {\If{$n'' \text{ dominates } n'$}
                \hfill \texttt{//See~\ref{subsec:Node domination}}
                \State valid = \textbf{False}
                \State \textbf{break}
            \EndIf}
        \EndFor}
        \vspace{1mm}
        \If{!valid}
            \hfill \texttt{//$n'$ was dominated}
            \State \textbf{continue}
        \EndIf  \label{line:dominates_end}
        \vspace{1mm}
        \For{$n'' \in$ OPEN with vertex~$v$} \label{line:subsuming_open1_init}
            {\If{$n'' \oplus n'$ is~$(\varepsilon,\kappa)-\text{bounded}$}
            \hfill \texttt{//See~\ref{subsec:Node subsuming}}
                \If{$n''$ has a better \ap coverage than~$n'$ } \label{line:better_coverage_open1}
                \State~$n'' \gets n'' \oplus n'$
                \State valid = \textbf{False}
                \State \textbf{break}
                \EndIf
            \EndIf}
        \EndFor
        \vspace{1mm}
        \If{!valid} 
            \hfill \texttt{//$n'$ was subsumed}
            \State \textbf{continue}
        \EndIf \label{line:subsuming_open1_end}
        \vspace{1mm}
        \For{$n'' \in$ OPEN with vertex~$v$} \label{line:subsuming_open2_init}
            {\If{$n' \oplus n''$ is~$(\varepsilon,\kappa)-\text{bounded}$}
            \hfill \texttt{//See~\ref{subsec:Node subsuming}}
            \If{$n'$ has a better \ap coverage than~$n''$ } \label{line:better_coverage_open2}
                \State OPEN.remove($n''$)
                    \hfill \texttt{//$n''$ was subsumed}
                \State~$n' \gets n' \oplus n''$
                \EndIf
            \EndIf}%
        \EndFor
    \EndFor \label{line:subsuming_open2_end}
    \vspace{1mm}
\State OPEN.insert$(n')$  \label{line:insert_to open}
\EndWhile

\State \Return NULL
\end{algorithmic}
\end{algorithm}

In the previous sections, we described how \irisuu modifies the search operations of \iris to account for localization uncertainty. In this section, we complete the description of the algorithm.
As we will see, despite these modifications, the high-level framework of \iris remains the same, and subsequently, its original guarantees, such as asymptotic convergence to an optimal solution.
This is done while also incorporating execution uncertainty and providing statistical guarantees (Sec.~\ref{sec:Theoretical guarantees}).
To this end, we proceed to outline an iteration of \irisuu given a tuple~$\langle \mathcal{G},m,\kappa,\varepsilon,\rho_{\text{coll}} \rangle$ (see Fig.~\ref{fig:AlgorithmicFlow}) whose pseudo-code is detailed in Alg.~\ref{alg:irisuu}.

In particular, similar to \iris's graph search, \irisuu uses a priority queue OPEN and a set CLOSED
while ensuring that all nodes are always \ek-bounded.
\irisuu starts with an empty CLOSED list and with an OPEN list initialized with the start node~$n_\text{init}$.
At each step, the search proceeds by iteratively popping a node~$n$ from the OPEN list whose \pap coverage~$\tilde{\rm{IPV}}_{u}$ is maximal.\footnote{As described in the original exposition of \iris, we can order the OPEN list either according to the \pap coverage or the \ap coverage.}
Then, if~$n$ satisfies the termination criteria~(see Eq.~\eqref{eq:goal_check}) we terminate the search and return~$n$'s command path~(Lines~\ref{line:check_goal_init}-\ref{line:check_goal_end}). 
Otherwise, we create a new node~$n'$ by extending~$n$ along its neighboring edges (Lines~\ref{line:extending_init}--\ref{line:extending_end}). However, if the estimated probability of collision for node~$n'$ exceeds a certain threshold~$\rho_\text{coll}$, then the newly created node~$n'$ is discarded (Lines~\ref{line:collision_init}--\ref{line:collision_end}).

If~$n'$ was not discarded, then, we perform the following operations (here we assume that the vertex corresponding with~$n'$ is~$v$):
\begin{itemize}
    \item We start by discarding~$n'$ if there exists a node~$n''$ in CLOSED that also reaches~$v$ and  dominates~$n'$~(Lines~\ref{line:dominates_init}-\ref{line:dominates_end}).
    \item We continue by testing whether there exists a node~$n''$ in OPEN that also reaches~$v$ that may subsume~$n'$, be \ek-bounded~(Lines~\ref{line:subsuming_open1_init}-\ref{line:subsuming_open1_end}) and has a better \ap coverage than~$n'$ (Line.~\ref{line:better_coverage_open1}).
    If so,~$n'$ is discarded and~$n''$ is set to be~$n'' \oplus n'$.
    \item Then, we test whether there exists a node~$n''$ in OPEN that also reaches~$v$ that may be subsumed by~$n'$ while the resultant node being \ek-bounded~(Lines~\ref{line:subsuming_open2_init}-\ref{line:subsuming_open2_end}), and the \ap coverage of~$n'$ is better than~$n''$ (Line ~\ref{line:better_coverage_open2}).
    If so,~$n''$ is removed from OPEN and~$n'$ is set to be~$n' \oplus n''$.
\end{itemize}

Finally, if~$n'$ was not discarded, it is inserted into the OPEN list~(Line~\ref{line:insert_to open}).

\textbf{Note.}
When there is no execution uncertainty, running \irisuu with $m=1$ is identical to \iris.


\section{\irisuu---Statistical guarantees}
\label{sec:Theoretical guarantees}
In this section, we detail in Sec.~\ref{subsec:Guarantees for a test path} different statistical guarantees regarding a given command path (proofs are provided in Appendix.~\ref{app:Lemma-proofs}).
Then, we  discuss the implication for the command path computed by \irisuu in Sec.\ref{subsec:Implication of statistical guarantees to irisuu} and provide guidelines on how to choose parameters for \irisuu given the statistical guarantees and the aforementioned implications.

\subsection{Guarantees for a given command path}
\label{subsec:Guarantees for a test path}
Consider a command path~$\pi$ and consider $m$ MC simulated executions of~$\pi$ such that 
$\hat{\rm{IPV}}(\pi) := \{ \hat{p}_1^{\pi}, \ldots, \hat{p}_k^{\pi} \}$ 
is the associated inspection probability vector,
$\hat{\mathcal{C}}(\pi)$ 
is the associated  estimated collision probability,
and
$\hat{\ell}(\pi)$ and~$\hat{s}_{\hat{\ell}}(\pi)$ are the associated average and standard deviation of the path's length, respectively. 
%

\vspace{2mm}
\begin{lemma}[Executed path's expected coverage]
\label{lemma:Bounding executed path's expected coverage}
%
For any desired CL of~$1 - \alpha\in [0,1]$, 
the expected coverage of an executed path following $\pi$, 
denoted by~$\vert \bar{\S}(\pi) \vert$, 
is at least:
\begin{equation}\label{eq:CI-coverage-lemma2}
    \vert \bar{\S}(\pi) \vert^- :=  \sum_{j=1}^{j=k}  \hat{p}^-(\hat{p}_j^{\pi},m,\alpha). 
\end{equation}
Here, the function~$\hat{p}^-$ 
is obtained from the Clopper-Pearson method~\cite{habtzghi2014modified}
and is defined in Eq.~\eqref{eq:Clopper-Pearson method} in Appendix~\ref{app:stat}.  
\end{lemma}

\ignore{
\vspace{2mm}
\begin{proof}
We treat the executed path's coverage probability as a random variable and recall that~$\hat{p}_j^{\pi}$ is the estimated probability to inspect POI~$j$ computed by~$m$ independent samples of execution paths. 
Then, for any desired CL of~$1 - \alpha$, the lower bound on the coverage is defined as:
$$
    \hat{p}{_j^{\pi}}^- =  \hat{p}^-(\hat{p}_j^{\pi},m,\alpha).
$$
Here, the function~$\hat{p}^-$ is defined in Eq.~\eqref{eq:Clopper-Pearson method}.
Since we assume that the inspection of each POI is independent  of the inspection outcomes of the other POIs (see sec~\ref{subsec:notations},
we can add up the lower bounds of the individual POIs to obtain a lower bound on the executed path's coverage.
Namely, we can say with a CL of {at least}~$1-\alpha$, a lower bound on the executed path's coverage is:
$$
    \vert \bar{\S}(\pi) \vert^-  
        =  \sum_{j=1}^{k}  \hat{p}{_j^{\pi}}^- 
        = \sum_{j=1}^{k} (\hat{p}_j^{\pi},m,\alpha).
$$
\end{proof}
}

\vspace{2mm}
\begin{lemma}[Executed path's collision probability]
\label{lemma:Bounding executed path's collision}
%
For any desired CL of~$1 - \alpha\in [0,1]$,
the expected collision probability of an executed path following~$\pi$, denoted by~$\bar{\mathcal{C}}(\pi)$ is at most: 
\begin{equation}\label{eq:lemma2-collision}
    \bar{\mathcal{C}}(\pi)^+ := \hat{p}^+(\hat{\mathcal{C}}(\pi),m,\alpha). 
\end{equation}
Here, the function~$\hat{p}^+$ is defined in Eq.~\eqref{eq:Clopper-Pearson method}. 
\end{lemma}

\ignore{

\vspace{2mm}
\begin{proof}
As we can treat the executed path's collision probability as a random variable,
Eq.~\eqref{eq:lemma2-collision} is an immediate application of the Clopper-Pearson method~(see Eq.~\eqref{eq:Clopper-Pearson method}).
\end{proof}
}

\vspace{2mm}
\begin{lemma}[Executed path's expected length]
\label{lemma:Bounding executed path's expected length}
%
For any desired CL of~$1 - \alpha\in [0,1]$, 
the expected length of an executed path following $\pi$, denoted by~$\bar{\ell}$ is bounded such that:
\begin{equation}
\label{eq:CI-length-lemma1}
    \bar{\ell}(\pi) \in \left[\bar{X}^-(\hat{\ell}(\pi),m,\alpha), \bar{X}^+(\hat{\ell}(\pi),m,\alpha) \right].
\end{equation}
Here,~$\bar{X}^-$ and~$\bar{X}^+$ are defined in Eq.~\eqref{eq:Mean of a population}.
\end{lemma}

Before stating our final Lemma, we introduce the following assumption:
\begin{assumption}
    \label{ass:mon-convex}
   For any fixed values of~$m$ and~$\alpha$, the function~$\hat{p}^-(\hat{p}_j^{\pi},m,\alpha)$, which depends solely on~$\hat{p}_j^{\pi}$, is both monotonically increasing and strictly convex.
\end{assumption}

\vspace{2mm} 
\begin{lemma}[\parbox{0.48\linewidth}{\centering {Bounding executed path's \\ sub-optimal coverage}}]
\label{lemma:Bounding executed path's sub-optimal coverage}
Recall that for any desired CL of $1-\alpha\in [0,1]$, $\vert \bar{\S}(\pi) \vert^-$ is the lower bound value of the expected coverage of an executed path following~$\pi$ and can be expressed as $\sum_{j=1}^{k} \hat{p}^-(\hat{p}_j^{\pi},m,\alpha)$ (see  Lemma~\ref{lemma:Bounding executed path's expected coverage} and Eq.~\eqref{eq:Clopper-Pearson method}).

If Assumption~\ref{ass:mon-convex} holds,  then minimizing $\vert \bar{\S}(\pi) \vert^-$ subject to 
(i)~$\hat{p}_j^{\pi} \in [0,1]$ for all $j$, 
and~(ii)~$\sum_{j=1}^{k} \hat{p}_j^{\pi} \geq \kappa \cdot k$
yields that~$\forall j,\hat{p}_j^{\pi} = \kappa$. 
Consequently, 
\begin{equation}
\label{eq:lemma_pj_kappa}
\min_{\hat{\rm{IPV}}(\pi)} \vert \bar{\S}(\pi) \vert^-
= k \cdot \hat{p}^-(\kappa,m,\alpha).
\end{equation}

As we will see, Assumption~\ref{ass:mon-convex} is used to prove Lemma~\ref{lemma:Bounding executed path's sub-optimal coverage}.
Proving that Assumption~\ref{ass:mon-convex} holds is non-trivial.
However, in Appendix~\ref{app:Lemma-proofs} we numerically demonstrate that it holds for all tested values.

\ignore{
%
\begin{color}{red}
  For a desired CL~$1 - \alpha$, with~$\alpha \in [0,1]$, and a parameter~$\kappa \in [0,1]$, we aim to minimize~$\vert \bar{\S}(\pi) \vert^-$ while satisfying the following conditions:~(i)~$\hat{p}_j^{\pi} \in [0,1]$ for all $j$, and~(ii)~$\sum_{j=1}^{k} \hat{p}_j^{\pi} \geq \kappa \cdot k$.
In the worst case, where the minimum lower bound is achieved, it follows that all elements of~$\hat{\rm{IPV}}(\pi)$ are equal and given by~$\forall j,\hat{p}_j^{\pi} = \kappa$. 
In other words:
\begin{equation}
\label{eq:lemma_pj_kappa}
\min_{\hat{\rm{IPV}}(\pi)} \vert \bar{\S}(\pi) \vert^-
= \sum_{j=1}^{k} \hat{p}^-(\hat{p}_j^{\pi},m,\alpha)
= k \cdot \hat{p}^-(\kappa,m,\alpha).
\end{equation}
\end{color}
}

\end{lemma}  

\ignore{
\vspace{2mm} 
\begin{proof}
    See proof in Appendix.\ref{app:Proof of lemma  sub-optimal coverage}.
\end{proof}
}

\subsection{Implication of statistical guarantees to \irisuu}
\label{subsec:Implication of statistical guarantees to irisuu}
One may be tempted to use the bounds on 
the expected coverage (Lemma~\ref{lemma:Bounding executed path's expected coverage})
and
the collision probability (Lemma~\ref{lemma:Bounding executed path's collision})
on the command path computed by \irisuu.
Indeed, these bounds hold if the estimations (e.g., the path's IPV) computed via MC simulated executions were computed \emph{after} the command path was computed by \irisuu and not on the fly \emph{while} the command path is computed by \irisuu.
That is, using these guarantees may lead to false negatives (e.g., estimation of a collision-free path despite the expectation of a collision occurring) since the command path is computed from a pool of multiple optional paths.
A detailed illustrative example to explain this is provided in Appendix~\ref{subsec:Illustrative example for possible false negatives}.

To summarize, the different statistical guarantees provided in Sec.~\ref{subsec:Guarantees for a test path} can be used if the command path outputed by  \irisuu is simulated multiple times (an alternative is to use the Bonferroni correction, which is a multiple-comparison correction used when conducting multiple dependent or independent statistical tests simultaneously~\cite{weisstein2004bonferroni}).
However, we can use them to understand the relationship between the system's parameters~$m$,~$\kappa$, and~$\rho_\text{coll}$ and as \emph{guidelines} on how to choose them. 

\vspace{2mm} 
\paragraph*{Lemma~\ref{lemma:Bounding executed path's sub-optimal coverage}---implications}
Recall that Lemma~\ref{lemma:Bounding executed path's sub-optimal coverage} states that for any desired CL, and regardless of the values of~$m$ and~$\kappa$, the minimum lower bound on the executed path's coverage is 
$
 \min_{\hat{\rm{IPV}}(\pi)} \vert \bar{\S}(\pi) \vert^- 
    = k \cdot \hat{p}^-(\kappa,m,\alpha)
$.
This allows us to provide guidelines on how to choose the algorithm's parameters~$m$ and~$\kappa$ according to the desired CL which is application specific.
As an example, in Fig.~\ref{subfig:CI lower bound of the POI coverage} we plot~$\hat{p}^-(\kappa,m,\alpha)$ for~$\alpha = 0.05$ for various values of~$m$ and~$\kappa$.
Now, consider a user requirement that the POI coverage of the executed path will exceed~$93\%$ with a CL of~$95\%$ (i.e.,~$\alpha = 0.05)$.
This corresponds to choosing any point on the line of~$93\%$ which can be, e.g.,~$m=70$ and~$\kappa = 0.99$ or~$m=95$ and~$\kappa=0.98$.

\vspace{2mm} 
\paragraph*{Lemma~\ref{lemma:Bounding executed path's collision}---implications}
Similar to Lemma~\ref{lemma:Bounding executed path's sub-optimal coverage}, Lemma~\ref{lemma:Bounding executed path's collision} can be used as a guideline on how to choose the algorithm's parameters~$m$ and~$\rho_\text{coll}$
according to the desired CL.
As an example, in Fig.~\ref{subfig:paretoCICollision} we plot~$\hat{p}^+(\rho_\text{coll},m,\alpha)$ for~$\alpha=0.05$ and various values of~$m$ and~$\rho_\text{coll}$.
Now, consider a user requirement that the executed path's collision probability does not exceed~$7\%$ with a CL of~$95\%$ (i.e.,~$\alpha=0.05$). This can be achieved by selecting a point on the~$7\%$ line, for instance,~$m=94$ and~$\rho_\text{coll}=0.02$.

\paragraph*{Lemma~\ref{lemma:Bounding executed path's expected length}---note}
Both Lemma~\ref{lemma:Bounding executed path's sub-optimal coverage} and 
Lemma~\ref{lemma:Bounding executed path's collision} gave clear guidelines on how to choose parameters for desired confidence levels.
This was possible because there exists a bound on the best possible outcome (i.e.,~$100\%$ coverage of POIs and~$0\%$ collision probability) and the system parameters~$\kappa$ and~$\rho_\text{coll}$ are defined with respect to these bounds.
In contrast, there is no a-priory bound on  path length and the parameter~$\eps$ is only defined with respect to the (unknown) optimal length~(i.e.,~$\hat{\ell}_u~\leq~(1+\varepsilon)~\cdot~\tilde{\ell}_u$).

\begin{figure}[tb]
    \centering
    \begin{subfigure}[tb]{0.235\textwidth}
        \centering
          \includegraphics[width=\textwidth]{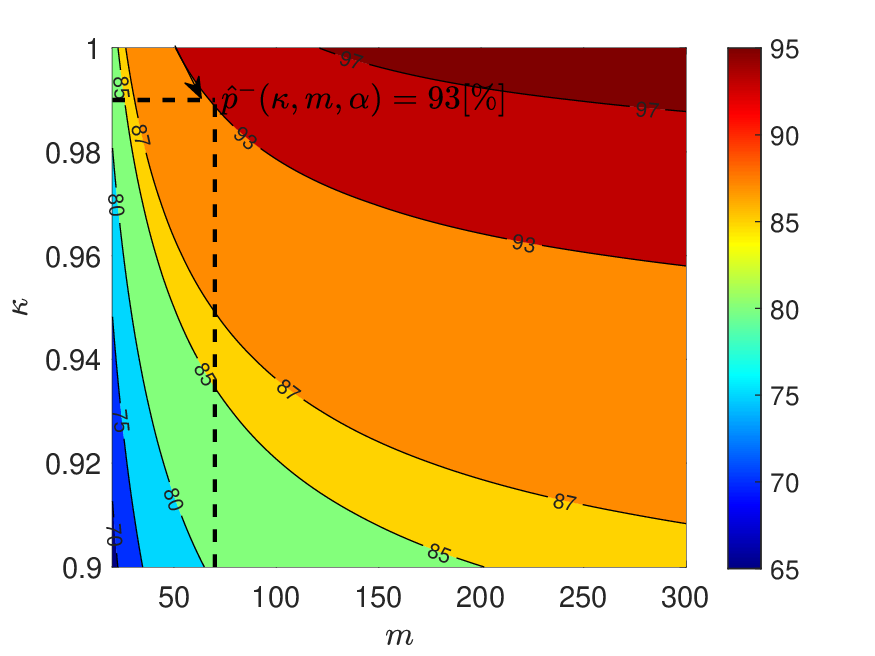}
        \caption{}
        \label{subfig:CI lower bound of the POI coverage}
    \end{subfigure}
    \begin{subfigure}[tb]{0.235\textwidth}
        \centering
        \includegraphics[width=\textwidth]{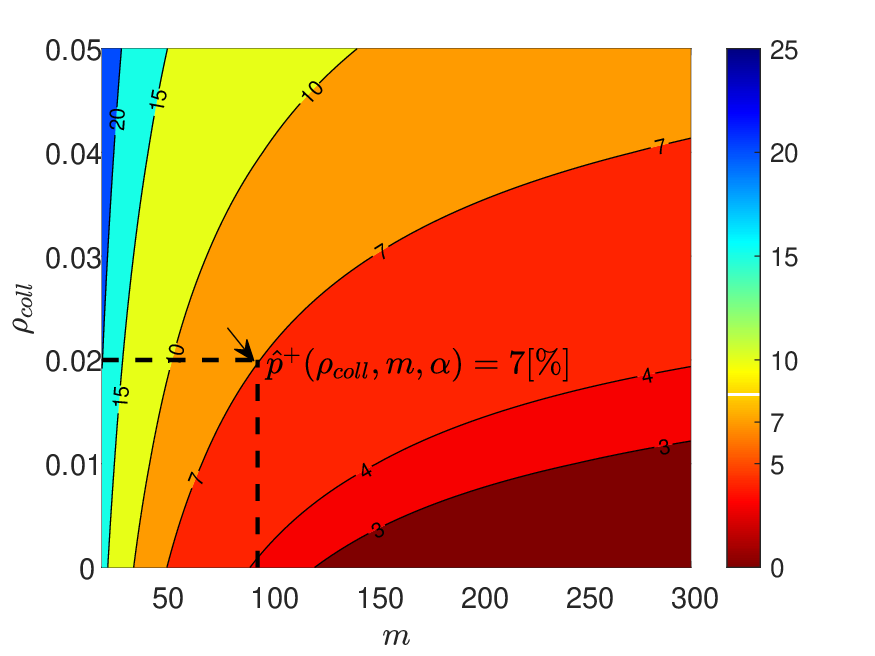}
        \caption{}
        \label{subfig:paretoCICollision}
    \end{subfigure}
    \caption{
    (\subref{subfig:CI lower bound of the POI coverage}), \protect (\subref{subfig:paretoCICollision})
    Values (in percentage) 
    of~$\hat{p}^-(\kappa,m,\alpha)$ 
    and~$\hat{p}^+(\rho_\text{coll},m,\alpha)$    
    for~$\alpha = 0.05$
    as a function of 
    the number of MC samples~$m$ ($x$-axis) and
    the coverage approximation factors~$\kappa$ and~$\rho_\text{coll}$
    ($y$-axis), respectively.
    }
    \label{fig:params}
\end{figure}

\section{Illustrative Scenario}
\label{Sec:DemoScenario}

In this section, we demonstrate the performance of \irisuu in a toy scenario using a simple motion model. 
We start (Sec.~\ref{subsce:toy scenario with a simple motion model}) by describing the setting 
and continue to describe the methods we will be comparing \irisuu with (Sec.~\ref{subsec:The compared methods}).
We  finish with a discussion of the results and their implications (Sec.~\ref{subsec:Result and implications of the toy scenario}).
All tests were run on an Intel(R) Core(TM) i7-4510U CPU @ 2.00GHz with 12GB of RAM.
The implementation \irisuu algorithm is available at 
\href{https://github.com/CRL-Technion/IRIS-UU.git}{https://github.com/CRL-Technion/IRIS-UU.git}.

\subsection{Setting}
\label{subsce:toy scenario with a simple motion model}
Here we consider the toy scenario depicted in Fig.~\ref{subfig:simpleScenario_illustration} in which we have a two-dimensional workspace that 
contains~$10$ obstacles (red rectangles) 
and~$27$ POIs (red points in seven groups of three).
The robot is described by three degrees of freedom---its location~$(x,y)$ and heading~$\psi$.
We model its sensor as having a field-of-view of~$94^\circ$ and a range of~$10[m]$.

We choose a roadmap~$\G$ with~$27$ vertices~(where vertex~$0$ is the initial vertex) such that the configuration associated with every vertex is facing up towards the POIs. 
We note that in the toy scenario, we hand-pick the roadmap only for illustrative purposes. In practical settings the roadmap would be generated by the systematic approach of  \iris  (as in the following section). 
The first three POIs can be seen from vertices~$0,9$ and~$18$.
Similarly, the next three POIs can be seen  can be seen from vertices~$1,10$ and~$19$ and so on.

Finally, we assume that the environment contains two types of regions corresponding to different levels of uncertainty (to be explained shortly): the first level (pink, containing vertices~$0-8$) with a standard deviation of~$\sigma = 3[m]$ and the second level (light blue, containing vertices~$9-26$) with~$\sigma = 1[m]$.

The motion model we use here, denoted as~$\Msimple^{2d}$, is an extension of~$\Mexample$ (Section~\ref{subsec:Running-example}) where the parameters $\Nloc:=(r,\theta)$ are drawn from the distribution $\Dloc$ such that~$r = |\mathcal{N}(0,1)|$ if the robot is located in low (pink) uncertainty region and $r = |\mathcal{N}(0,3)|$ if the robot is located in high (light blue) uncertainty region. Additionally~$\theta = \mathcal{N}(0,2\pi)$.
Just as in~$\Mexample$, given a command path~$\pi^c$ with $n>1$ configurations~$
\pi^c = \{ \langle x_1^c,y_1^c \rangle, \ldots \langle x_n^c,y_n^c \rangle \}
$,
the corresponding executed path is~$\pi_e = 
\{ \langle x_1^e,y_1^e \rangle, \ldots \langle x_n^e,y_n^e \rangle \}
$
such that~$x_i^e = x_i^c + r \cdot \cos \theta$ and~$y_i^e = y_i^c + r \cdot \sin \theta$ for $i>1$.

\ignore{
\begin{figure}[tb]
    \centering
     \includegraphics[width=0.5\textwidth]{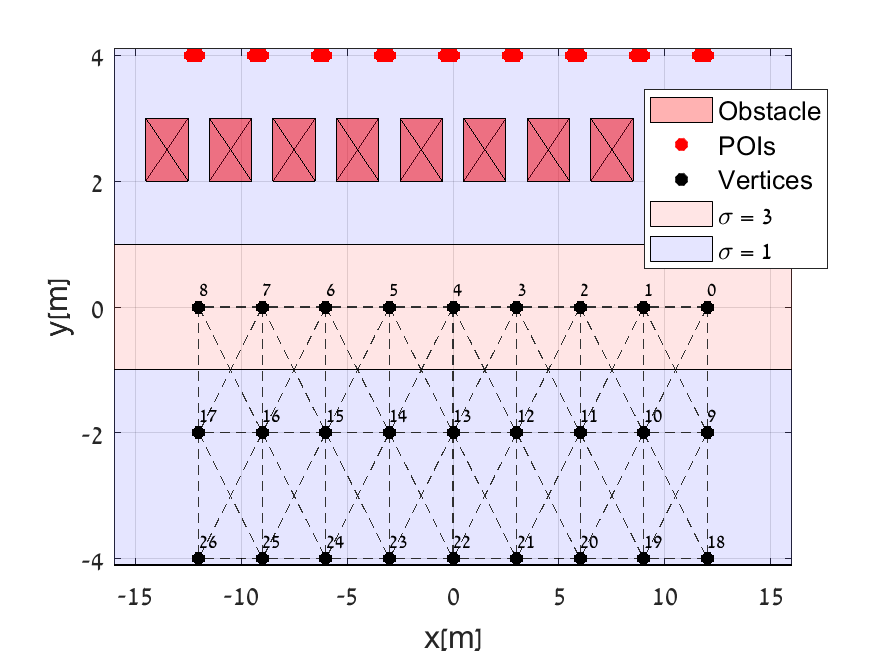}
    \caption{
    Toy scenario for empirical evaluation.}
     \label{fig:simpleScenario_illustration}
\end{figure}
}

\begin{figure}[tb]
    \centering
    \begin{subfigure}[tb]{0.235\textwidth}
        \centering
        \includegraphics[width=\textwidth]{figs/SimpleScenario.eps}
        \caption{}
        \label{subfig:simpleScenario_illustration}
    \end{subfigure}
    \hfill
    \begin{subfigure}[tb]{0.235\textwidth}
        \centering
        \includegraphics[width=\textwidth]{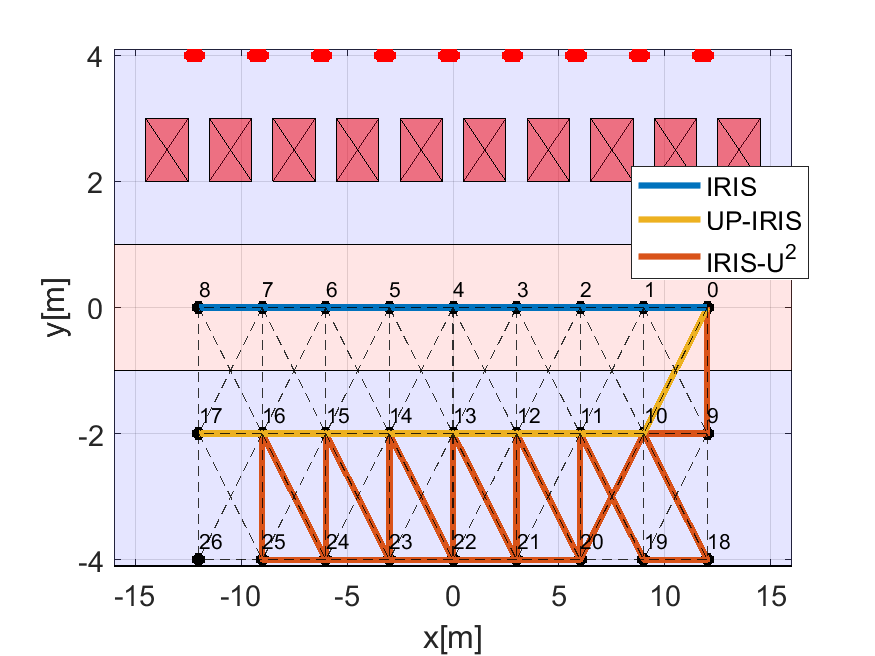}
        \caption{}
        \label{subfig:toy scenario path compared methods}
    \end{subfigure}
    \caption{
    (\subref{subfig:simpleScenario_illustration})
    Illustrative scenario for empirical evaluation depicting a roadmap, (black points being vertices and dashed lines being edges), POIs (red points), obstacles (red rectangles) and regions with high and low uncertainty.
    (\subref{subfig:toy scenario path compared methods})~Command paths computed by \iris (blue), \lum (yellow), and \irisuu (orange) with~$m=100$.
    }
\end{figure}

\ignore{
In our simple motion model denoted as~$\Msimple$, we assume that the path~$\pi^e$ obtained by executing a command path~$\pi^c$
is defined as follows:
there is uncertainty with respect to the location of~$\pi^e$'s vertices but not with regards to their heading.
In particular, let~$d_{x_i} = r \cdot \cos \theta$ and
$d_{y_i} = r \cdot \sin \theta$ be the location error of specific configuration~$i$, where,~$r\sim\mathcal{N}(0,\sigma)$, and~$\theta \sim \mathcal{N}(0,2\pi)$.
Here~$\sigma$ is equal~$1$ or~$3$ depending on the vertex location.
Then, the location uncertainty of the execution path~$ \{ \langle x_1^e,y_1^e \rangle; \langle x_2^e,y_2^e \rangle; \ldots \}$ of following the command path~$ \{ \langle x_1^c,y_1^c \rangle; \langle x_2^c,y_2^c \rangle; \ldots \}$ is defined such that~$x_i^e = x_i^c + d_{x_i}$ and~$y_i^e = y_i^c + d_{y_i}$.
}

\subsection{Baselines}
\label{subsec:The compared methods}
We consider two baselines---the original \iris algorithm and a straw-man approach which we call uncertainty-penalizing \iris (\lum).
In \lum, the cost of an edge is its length added to a penalty factor which is proportional to the uncertainty associated with the edge.
As the original \iris minimizes path length, this modification will compute paths that are both short and have low uncertainty.

For each algorithm, we compute a command path~$\pi^c$ and compare performance in the execution phase.
In preparation for the results, we highlight the distinction between MC samples used in planning and in execution. 
Planning MC samples are used by \irisuu to compute the command-path (referred to as~$m$). 
In contrast, execution MC samples do not affect the command-path and are only used to evaluate the performance of the executed path and \irisuu's CI boundaries.

\begin{figure*}[t]
    \centering
    \begin{subfigure}[tb]{0.49\textwidth}
        \centering
        \includegraphics[width=\textwidth]{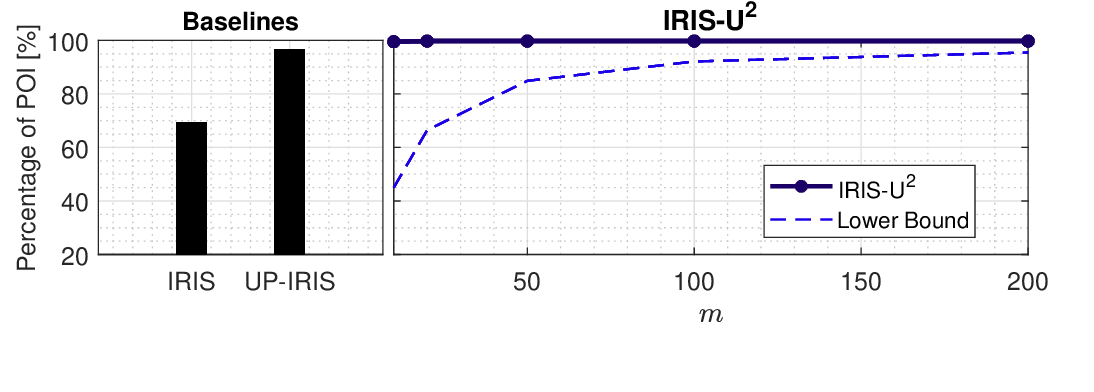}
        \caption{}
        \label{fig:sub1-poi-coverage-toy}
    \end{subfigure}
    \hfill
    \begin{subfigure}[tb]{0.49\textwidth}
        \centering
        \includegraphics[width=\textwidth]{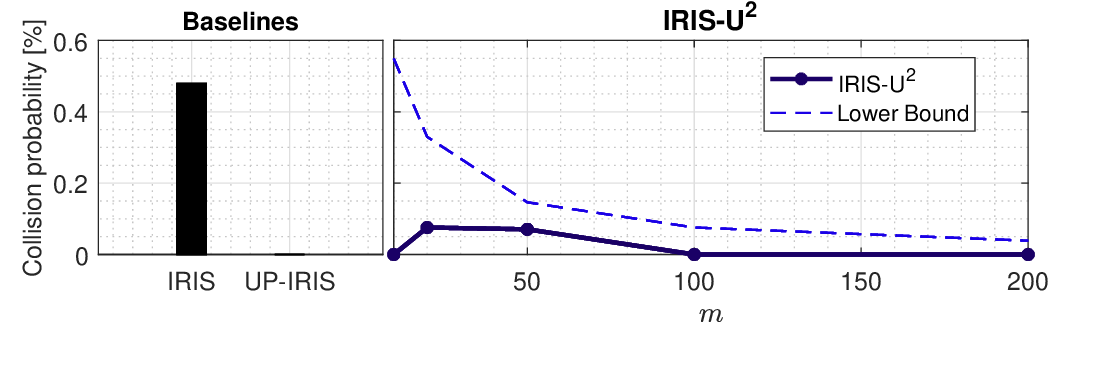}
        \caption{}
        \label{fig:sub2-coll-prob-toy}
    \end{subfigure}

    \vskip\baselineskip

    \begin{subfigure}[tb]{0.49\textwidth}
        \centering
        \includegraphics[width=\textwidth]{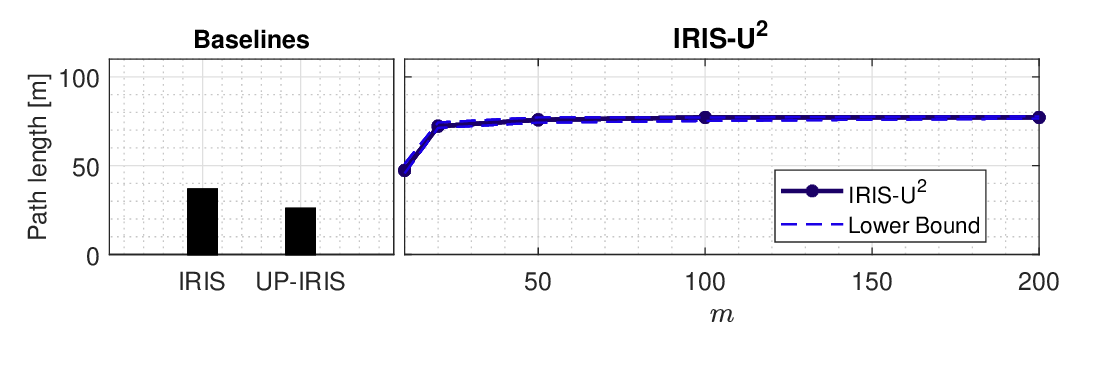}
        \caption{}
        \label{fig:sub3-path-length-toy}
    \end{subfigure}
    \hfill
    \begin{subfigure}[tb]{0.49\textwidth}
        \centering
        \includegraphics[width=\textwidth]{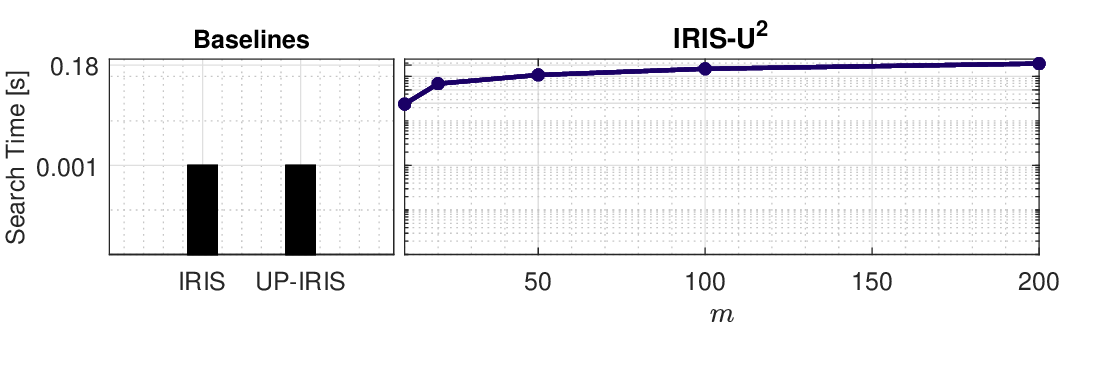}
        \caption{}
        \label{fig:sub4-search-time-toy}
    \end{subfigure}
    \vskip\baselineskip
    \caption{
    \protect 
    (\subref{fig:sub1-poi-coverage-toy})-\protect (\subref{fig:sub4-search-time-toy})
    POI coverage, collision probability,  path length, and search time as a function of $m$
    for \irisuu (right) and for the baselines  (left) for the illustrative scenario.
    Each of figures \protect (\subref{fig:sub1-poi-coverage-toy})-\protect (\subref{fig:sub3-path-length-toy}) displays the average of~$10,000$ MC execution samples, with the corresponding CI of the expected performance as detailed in Lemma.~\ref{lemma:Bounding executed path's expected coverage},~\ref{lemma:Bounding executed path's collision}, and~\ref{lemma:Bounding executed path's expected length}.
    }
    \label{fig:four_subfigures}
\end{figure*}

\subsection{Results}
\label{subsec:Result and implications of the toy scenario}
To compare \irisuu with our baselines, \iris and \lum, we ran the planning phase with~$\varepsilon=3$ and~$\kappa=0.99$. 
As expected, \iris (without accounting for uncertainty) computed a command path which is the straight line connecting vertices~$0$ and~$8$ (blue path in Fig.~\ref{subfig:toy scenario path compared methods}) as it is the shortest path that results in full coverage (when ignoring uncertainty).
\lum  on the toy scenario computes a command path (depicted as the yellow path in Fig.~\ref{subfig:toy scenario path compared methods}) which prioritizes regions with low localization uncertainty. This path moves towards vertex~$10$ and then proceeds in a straight line towards vertex~$17$ while staying within the low-uncertainty region.
Similarly, the command path computed by \irisuu (represented by the red path in Fig.~\ref{subfig:toy scenario path compared methods}) also prioritizes regions with low localization uncertainty, even though this is not its explicit objective. However, instead of simply traversing through the low-uncertainty region once, \irisuu revisits multiple vertices to ensure full coverage.

Numerical results concerning these command paths in the execution phase are depicted in Fig.~\ref{fig:four_subfigures}.
\iris and \lum
achieved an average coverage of~$69\%$ and~$96\%$, respectively. 
Both being lower than the at least~$99\%$ achieved by  \irisuu for all values of $m$.
Additionally, the collision probability of \iris's path was found to be~$47\%$ as opposed to the collision-free execution path of \lum and \irisuu (for $m\geq 100$).

When looking at the performance of \irisuu as a function of the number of MC samples~$m$, one can see (as to be expected) that the expected POI covered increases up to~$99\%$ and that the collision probability decreases down to~$0\%$ as~$m$ increases
(Fig.~\ref{fig:sub1-poi-coverage-toy} and~\ref{fig:sub2-coll-prob-toy}, respectively).
This comes at the price of longer paths and longer computation times 
(Fig.~\ref{fig:sub3-path-length-toy} and~\ref{fig:sub4-search-time-toy}, respectively).

Note that  in all cases, the CI bounds hold empirically.
This is important as our analysis relies on the  fact that the probability of inspecting each POI is independent of other POIs (Assumption~\ref{ass:iid}) which may not necessarily hold.

\ignore{
\begin{figure}[tb]
    \centering
     \includegraphics[width=0.5\textwidth]{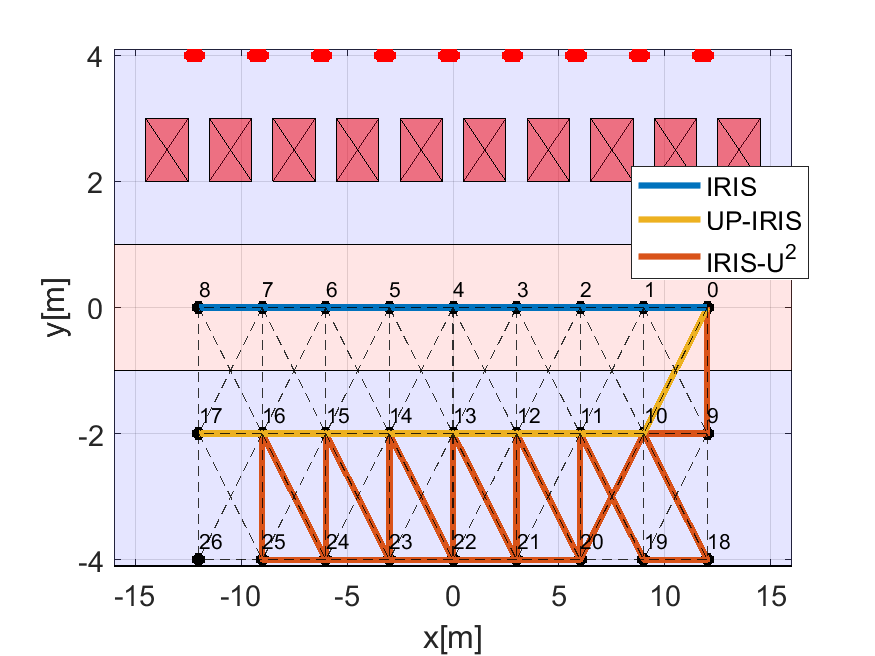}
     \caption{Command paths computed by \iris (blue), \lum (yellow) and \irisuu (orange) in the toy scenario.}
\label{fig:toy scenario path compared methods}
\end{figure}
}


\section{Bridge Scenario}
\label{Sec:BridgeScenario}


\check{In the following set of experiments, we consider a realistic bridge-inspection scenario to assess the performance of \irisuu. 
The evaluation consists of two parts wherein both parts use a well-established approximation of the full motion model, which we call the \emph{simplified model}, in the planning stage. 
In the first part, we use the same simplified motion model to evaluate the command path in the execution phase. This experiment demonstrates our statistical bounds (which hold under the assumption that the planning-stage model is accurate). 
In the second part, we use a more accurate model to evaluate the command path in the execution phase. In this model, which is much more computationally demanding, uncertainty accumulates in GNSS-denied regions. This part aims to demonstrate the performance of our algorithm on a more realistic motion model and  explores the implications of having a mismatch between the planning and execution models (both of which accumulate errors). In our setting, the mismatch arises due to computational considerations (see discussion below). We first describe below the scenario and then proceed to the experimental results for each of the motion models.}

\subsection{Setting}
We consider a UAV with six degrees of freedom corresponding to its location~$(x,y,z)\in \mathbb{R}^3$ and its orientation~$(\phi,\theta,\psi)\in \mathbb{R}^3$.
We model the sensor as having a field-of-view of~$94^\circ$ and a range of~$10m$.
We use the 3D model of a bridge\footnote{Model taken from~\href{https://github.com/UNC-Robotics/IRIS}{https://github.com/UNC-Robotics/IRIS}} as depicted in Fig.~\ref{fig:bridgeScenario} and set the values~$\varepsilon=3$ and~$\kappa=0.9$. 
Following~\cite{DBLP:journals/ijrr/FuKSA23} the roadmap~$\G$ was generated using an RRG~\cite{Karaman2011_IJRR} with~$100$ vertices and~$1224$ edges. 

\begin{figure}[tb]
          \includegraphics[width=0.5\textwidth]{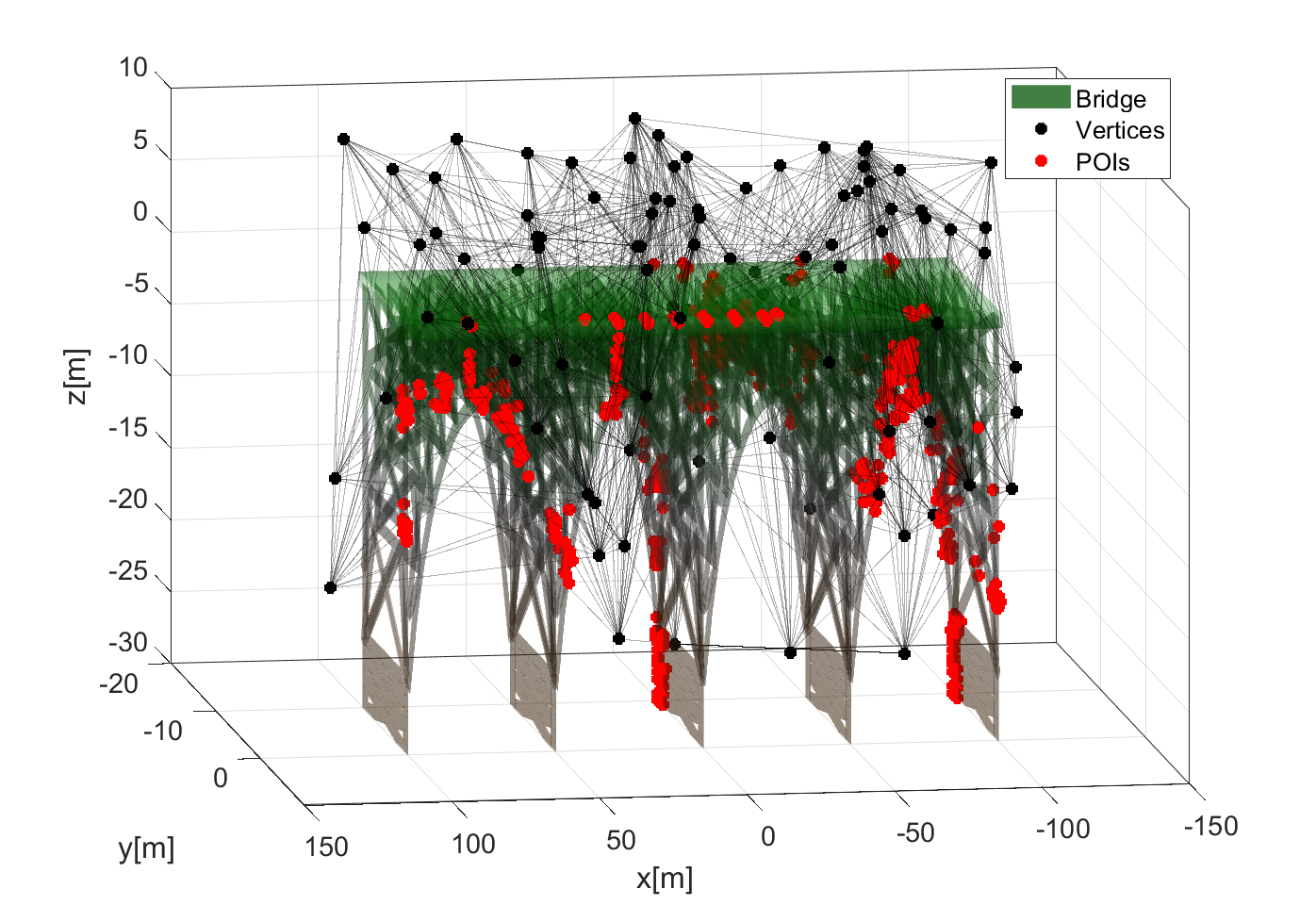}
        \caption{ Bridge scenario with~$533$ POIs (marked in red). 
        Search is performed on a roadmap~$\G$ (marked in black).
        \label{fig:bridgeScenario}
   }    
\end{figure}

\ignore{
\paragraph{Simplified motion model}
We extend the simple motion model~$\Msimple^{2d}$ described in Section~\ref{subsce:toy scenario with a simple motion model} to the 3D setting, which we denote by $\Msimple^{3d}$.
Namely, now we sample normally in a sphere around the given location.
Specifically, the position error of a specific configuration~$(x_i,y_i,z_i,\phi_i,\theta_i,\psi_i)\in \mathbb{R}^6$ is given by~$d_{x_i} = r \cdot \cos \theta \cos \psi$,~$d_{y_i} = r \cdot \cos \theta \sin \psi$, and~$d_{z_i} = -r \sin \theta$, where~$r \sim |\mathcal{N}(0,\sigma)|$,~$\theta \sim \mathcal{N}(0,2\pi)$, and $\psi = \mathcal{N}(0,2\pi)$, where the value~$\sigma$ is location dependent.
The location uncertainty of an execution path,~${ \langle x_1^e,y_1^e,z_1^e \rangle; \langle x_2^e,y_2^e,z_1^e \rangle; \ldots }$, when following the command path,~${ \langle x_1^c,y_1^c,z_1^c \rangle; \langle x_2^c,y_2^c,z_2^c \rangle; \ldots }$, is defined such that (i)~$x_i^e = x_i^c + d_{x_i}$, (ii)~$y_i^e = y_i^c + d_{y_i}$, and (iii)~$z_i^e = z_i^c + d_{z_i}$.
We set~$\sigma$  to be either~$1$ or~$3$ corresponding to low and high uncertainty regions, respectively.
Specifically, the high uncertainty region is motivated by GNSS outages which typically occur under bridges.
}

\paragraph{Exact motion model}
We make use of a highly realistic model, which we denote by $\Mexact$ which requires fusing measurements from the GNSS and the inertial navigation system~(INS) via an EKF~\cite{bookGroves}, known as inertial navigation system GNSS-INS fusion.
%
\check{Our implementation\footnote{See \url{https://github.com/CRL-Technion/Simulator-IRIS-UU.git}} of $\Mexact$ which uses an EKF is based on the quadcopter dynamics of ArduCopter\footnote{See \url{https://wilselby.com/research/arducopter/}}, an open-sourced quadrotor system with an adaptation of the EKF fusion.}
%
%
\paragraph{Simplified motion model}
During the planning stage, we consider a computationally efficient and well-established approximation model from~\cite{bookGroves}, which we denote by $\Msimplified$. Indeed, our experimental results below demonstrate that~$\Msimplified$ serves as a good proxy for $\Mexact$ as the predicted behavior during the planning stage correlates with the execution behavior.
Specifically, in our model~$\Msimplified$, we assume that outside the bridge, the GNSS-INS system can achieve an accuracy with an error of~$1\sigma$ (in meters) around a vertex location.
In contrast, beneath the bridge, where GNSS signal reception is typically compromised, we assume accumulating uncertainty over time.
For this region, we adopt a model~\cite{bookGroves} that assumes that the robot moves at a constant speed in a straight line between nodes. Specifically, the location uncertainty of an execution path,~${ \langle x_1^e,y_1^e,z_1^e \rangle; \langle x_2^e,y_2^e,z_1^e \rangle; \ldots }$, when following the command path,~${ \langle x_1^c,y_1^c,z_1^c \rangle; \langle x_2^c,y_2^c,z_2^c \rangle; \ldots }$, can be expressed as follows:
\begin{equation}\label{eq:qt_wden}
[x_i^e,y_i^e,z_i^e] = [x_i^c,y_i^c,z_i^c] +\frac{1}{2}\bm{C_{b}} \vec{b_a} t^2 + \frac{1}{6}\bm{C_{b}} (\vec{b_g} \times \vec{g}) t^3.
\end{equation}
Here,~$\vec{g} = \left( 0,0,-g \right)$ is the gravity vector,~$\vec{b_a}$ and ~$\vec{b_g}$ are the accelerometer and gyro biases. 
Additionally,~$C_{b}$ denotes the rotation matrix transforming from the body frame to the inertial frame, and~$t$ denotes the continuous time spent in GNSS outage regions (beneath the bridge).

\subsection{Evaluation}
\check{Similar to Sec.~\ref{Sec:DemoScenario}, we compare \irisuu with our two baselines \iris and \lum. However, for each of the three algorithms, we consider the two different motion models in the execution phase.
Thus, we use \textsc{ALG-SEM} to refer to the setting where an algorithm $\textsc{ALG} \in \{ \irisuu, \iris, \lum\}$ 
uses~$\Msimplified$ both in the planning and in the execution phase (here, `SEM' refers to Simplified Execution Model).
Similarly, we use \textsc{ALG-EEM} to refer to the setting where an algorithm $\textsc{ALG} \in \{ \irisuu, \iris, \lum\}$ 
uses  $\Msimplified$ in the planning phase while using $\Mexact$ in the execution phase (here, `EEM' refers to Exact Execution Model).
}

\check{Fig.~\ref{fig:DEMOBridge_scenario_Result} contains the coverage, collision probability, path length and search time for the different algorithms.}

\begin{figure*}
    \centering
    \begin{subfigure}[tb]{0.49\textwidth}
        \centering
        \includegraphics[width=\textwidth]{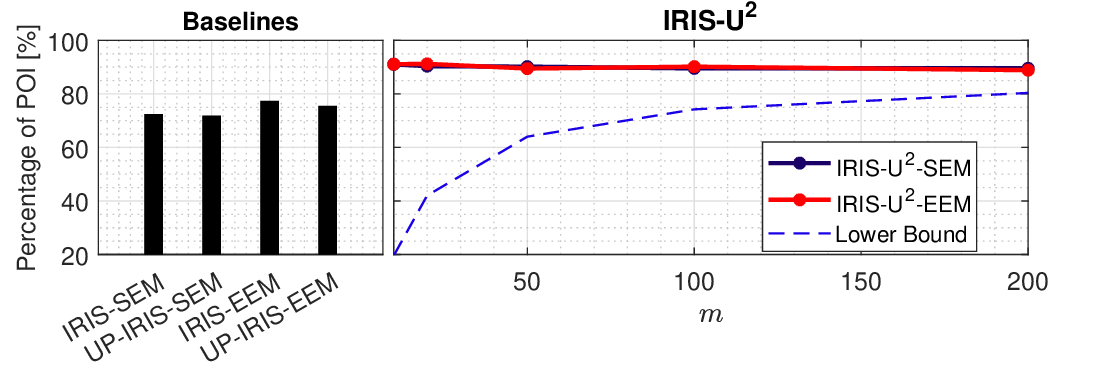}
        \caption{}
        \label{fig:DEMOBridgesub1-poi-coverage}
    \end{subfigure}
    \hfill
    \begin{subfigure}[tb]{0.49\textwidth}
        \centering
        \includegraphics[width=\textwidth]{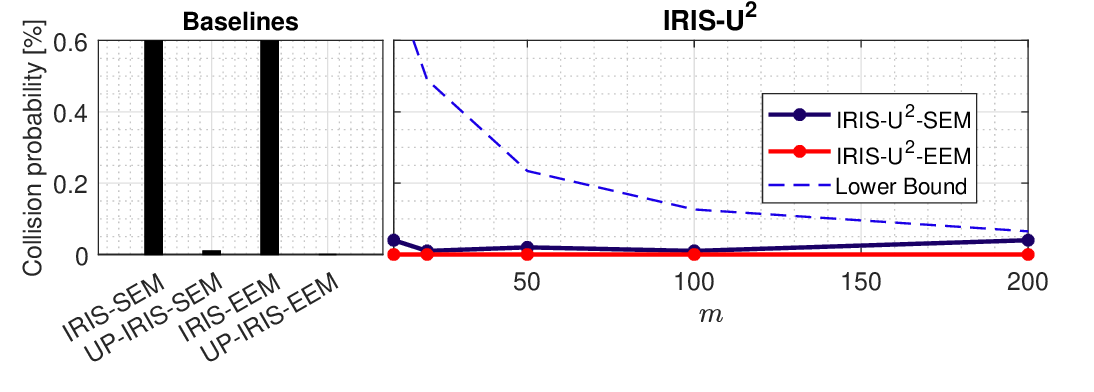}
        \caption{}
        \label{fig:DEMOBridgesub2-coll-prob}
    \end{subfigure}
    
    \vskip\baselineskip

    \begin{subfigure}[tb]{0.49\textwidth}
        \centering
        \includegraphics[width=\textwidth]{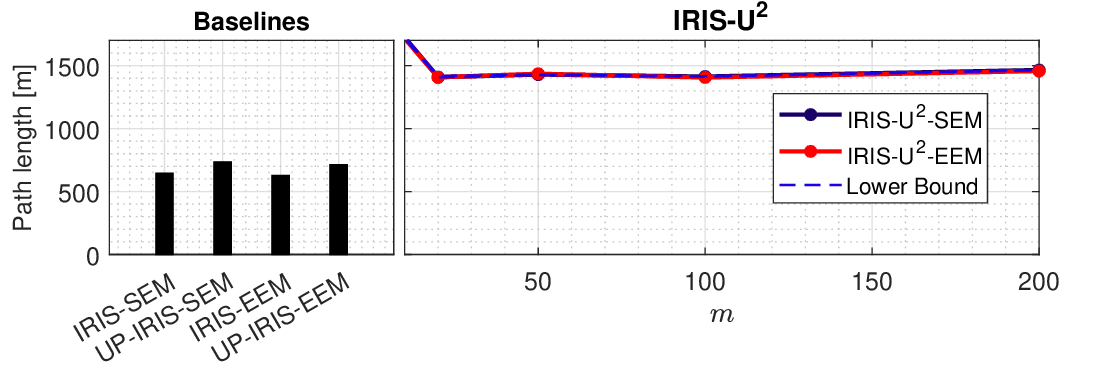}
        \caption{}
        \label{fig:DEMOBridgesub3-path-length}
    \end{subfigure}
    \hfill
    \begin{subfigure}[tb]{0.49\textwidth}
        \centering
        \includegraphics[width=\textwidth]{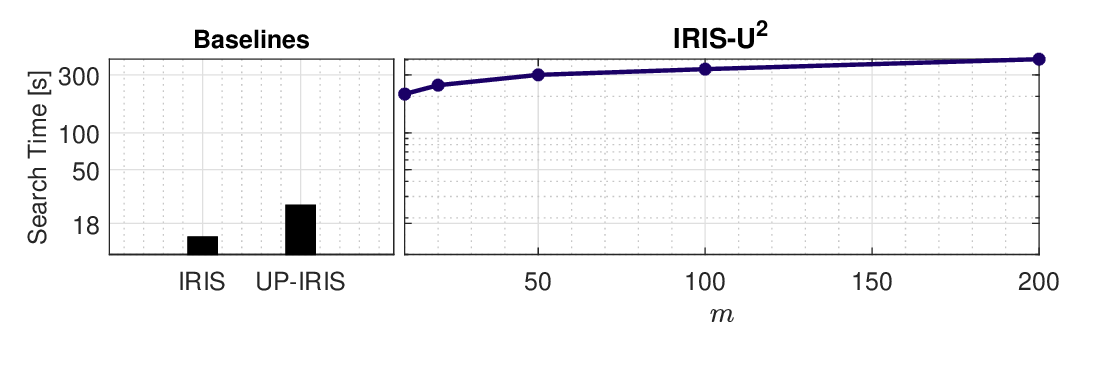}
        \caption{}
        \label{fig:DEMOBridgesub4-search-time}
    \end{subfigure}
       \vskip\baselineskip
    \caption{
    \protect Experimental results for the bridge scenario. (\subref{fig:DEMOBridgesub1-poi-coverage})-\protect (\subref{fig:DEMOBridgesub4-search-time})
    POI coverage, collision probability,  path length, and planning search time as a function of $m$
    for \irisuu (right) both with the baseline performance of \iris and \lum (left).
    Each of figures \protect (\subref{fig:DEMOBridgesub1-poi-coverage})-\protect (\subref{fig:DEMOBridgesub3-path-length}) displays the average of~$100$ MC execution samples (blue dots), corresponding to the CI of the expected performance (red line) as detailed in Lemma.~\ref{lemma:Bounding executed path's expected coverage},~\ref{lemma:Bounding executed path's collision}, and~\ref{lemma:Bounding executed path's expected length}. 
    }
    \label{fig:DEMOBridge_scenario_Result}
\end{figure*}

\check{
\subsubsection{Planning and execution with the same model}}
\check{
\iris achieved an average coverage of~$72\%$, which is significantly lower than the~$91\%$ coverage (more than the value of~$\kappa=0.9$)  achieved by \irisuu even when using only~$m=10$.
Additionally, the path computed by \iris was found to be in collision in $60\%$ of the time as opposed to \irisuu whose path was found to be less than~$4\%$ collision in all tested execution paths.
}

\check{
In the case of \lum, despite its shorter calculation time and the fact that the command path has a low collision probability, it achieved an average coverage of only $71\%$.
This highlights the critical significance of considering uncertainty not only in terms of localization uncertainty, which impacts collision probability and path length but also in projecting its effects on the primary objective of inspecting the POIs. 
}

\check{Finally, note that all statistical guarantees stated in Sec.~\ref{sec:Theoretical guarantees} hold (as expected).}

\check{\subsubsection{Planning and execution with different models}}
\check{
One may expect that paths computed when planning with a simplified model can be very poor in quality (coverage, collision probability and length) when evaluated with an exact motion model.
Despite~$\Msimplified$ being an approximation of~$\Mexact$, results across all algorithms differ only slightly.
}

\check{More importantly, as we use a different execution model from the one used when planning, CI bounds are not guaranteed to hold. However, we can see empirically that all results  \irisuu-\textsc{EEM} falls within the confidence bounds computed for \irisuu-\textsc{SEM}.}

\ignore{
Note that as we use a different execution model from the one used when planning, CI bounds do not necessarily hold.
The CI bound of \irisuu in this scenario may not accurately reflect the true uncertainty during the execution phase. 
Indeed, as depicted in Fig.~\ref{fig:Bridge_scenario_Result}, this results in disparities between the executed path length and the CI bounds of \irisuu.
However, despite using a simplified model, \irisuu is effective in achieving high execution coverage, with CI bounds that appropriately encapsulate this coverage. 

Furthermore, search time is longer here than for the simple motion model as the position error in GNSS outages in this setting is accumulated over time. 
As a result, the performance (i.e., coverage, collision, and path length) at each vertex in these regions depends on the path history and requires more complex updates.
However, we can use the same approach of reducing the search time by using the guidelines described in Sec.~\ref{subsec:Implication of statistical guarantees to irisuu}.
We plot in Fig.~\ref{fig:sub_optimal_bridge_equivalent_kappa} the search time for different values of~$\kappa$ and~$m$ and see similar trends to the results described in Sec.~\ref{subsec:Planning with simple motion model}.
Specifically, by accurately balancing $m$ and~$\kappa$, we can obtain a speedup of up to $\times 4$ and get a path whose coverage is still above~$\hat{p}_{\text{desired}}^-$.
}

\subsubsection{Sensitivity analysis of parameters to runtime}

\begin{figure}[tb]
    \centering
    \includegraphics[width=0.4\textwidth]{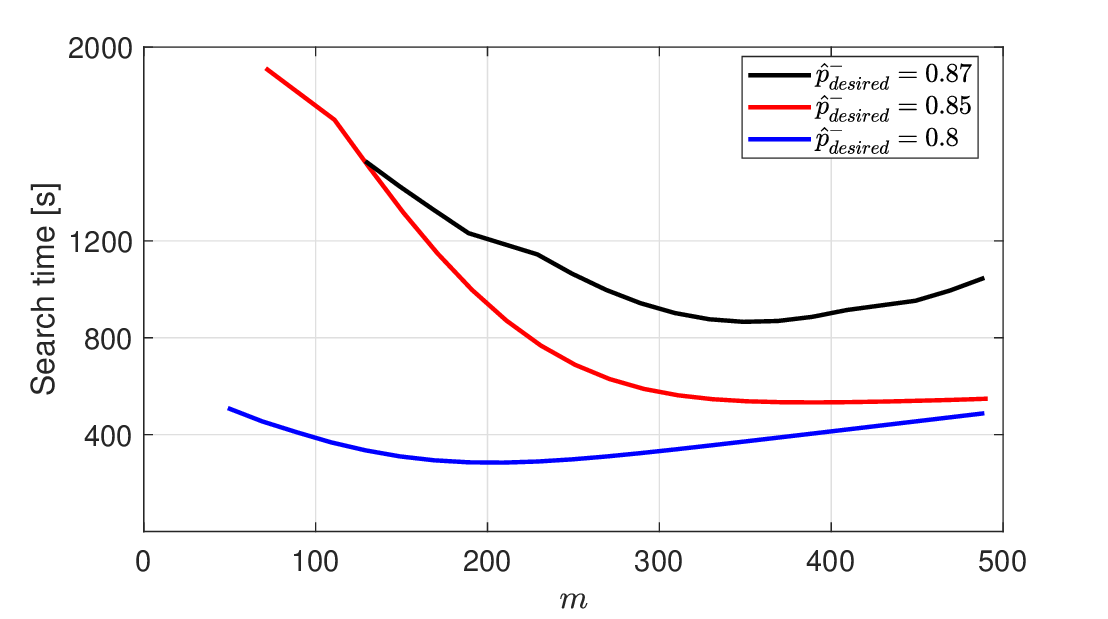}
  \caption{ 
  Search time in the planning phase 
  in the bridge scenario using the simple motion model $\Msimple^{3d}$ as a function of $m$ for several values of~$\hat{p}^-_{\text{desired}}$.
  Here, for each value of~$m$,~$\kappa$ is the smallest value for which~$\hat{p}^-(\kappa,m,\alpha) \geq \hat{p}^-_{\text{desired}}$ (see Fig.~\ref{fig:params}). 
  }
  \label{fig:sub_optimal_demo_bridge_equivalent_kappa_fix}
\end{figure}

To obtain better execution coverage and fewer collisions, \irisuu typically requires longer planning times~(see Fig.~\ref{fig:DEMOBridgesub4-search-time}).
However, as described in Sec.~\ref{subsec:Implication of statistical guarantees to irisuu}, for a desired confidence level, we can choose between several parameters.
We plot in Fig.~\ref{fig:sub_optimal_demo_bridge_equivalent_kappa_fix} the search time for different values of~$\kappa$ and~$m$.
Roughly speaking, increasing~$m$ (and thus decreasing~$\kappa$) reduces computation time. 
This is because higher values of $\kappa$ do not allow \irisuu to subsume nodes and the computational price of maintaining more nodes is typically larger than using more MC samples. However, after a certain number of MC samples is reached, this trend is reversed.
This trade-off is dramatic, 
for example, when considering a desired value of~$\hat{p}^-= 0.85$ 
   (i.e., at least~$85\%$ POIs will be covered for a CL of $1-\alpha = 0.95$), 
the planning times range 
from~$1,912$ seconds for $\langle m, \kappa \rangle  = \langle 71,0.934 \rangle$
to~$533$ seconds for $\langle m, \kappa \rangle  = \langle 371,0.885 \rangle$.

\subsubsection{Summary}
We summarize this section with a high-level comparison of the different approaches, also visualized in the accompanying \href{https://crl.cs.technion.ac.il/wp-content/uploads/2024/04/IRIS-UU-SuplumentaryVideo-compressed.mp4}{video}.
As \iris does not account for uncertainty, the executed path tends to miss POIs and may even collide with the bridge.
\lum, on the other hand, prioritizes low-uncertainty regions and hence typically yields a collision-free executed path, albeit it can still miss POIs.
Finally, the planned path of \irisuu is typically longer, often traversing edges several times to ensure that no POIs are missed. This is done while ensuring that the path is collision free, even in regions with high uncertainty.

\section{Conclusion and future work}
\label{Sec:Conclusion and future work}

In this study, we proposed \irisuu, an extension of the \iris offline path-planning algorithm that considers execution uncertainty through the use of MC sampling. 
%
Our empirical results demonstrate that \irisuu provides better performance under uncertainty in terms of coverage and collision while providing statistical guarantees via confidence intervals. In addition, we provide a guideline on how to choose parameters to reduce the computation time based on the statistical guarantees.

However, as we discussed in Sec.~\ref{subsec:Implication of statistical guarantees to irisuu}, the guarantees on the execution path could be affected by a false-negative bias, particularly when there are multiple optional paths in the planning process. As a promising direction, we suggest to explore the Bonferroni correction method~\cite{weisstein2004bonferroni}, and combine it with the information about the number of optional paths considered during planning. By doing so, we can strengthen our CI bounds and enhance the robustness and reliability of our guarantees.

Additionally, recall that we made two key assumptions in this work:
(i)~that inspecting different POIs is i.i.d  (Assumption~\ref{ass:iid})
and
(ii)~that we wish to minimize path length and not energy consumption or mission completion time (see note following Prob.~\ref{prob-1}).
In future work, we plan to relax both assumptions.

Another direction for future research involves exploring alternative methods to decrease search time.
One promising approach is to reduce the number of MC samples by implementing alternative sampling methods such as \emph{Latin hypercube sampling}~(LHS)~\cite{loh1996latin}, which better distributes samples across the parameter space.
The number of samples can be further reduced by leveraging information from the covariance matrix of the EKF (see, e.g.,~\cite{papachristos2019autonomous}), particularly when navigation sensors uniformly cover the entire uncertainty region.

\bibliographystyle{IEEEtran}       
\bibliography{references}

\begin{thebibliography}{10}
\providecommand{\url}[1]{#1}
\csname url@samestyle\endcsname
\providecommand{\newblock}{\relax}
\providecommand{\bibinfo}[2]{#2}
\providecommand{\BIBentrySTDinterwordspacing}{\spaceskip=0pt\relax}
\providecommand{\BIBentryALTinterwordstretchfactor}{4}
\providecommand{\BIBentryALTinterwordspacing}{\spaceskip=\fontdimen2\font plus
\BIBentryALTinterwordstretchfactor\fontdimen3\font minus \fontdimen4\font\relax}
\providecommand{\BIBforeignlanguage}[2]{{%
\expandafter\ifx\csname l@#1\endcsname\relax
\typeout{** WARNING: IEEEtran.bst: No hyphenation pattern has been}%
\typeout{** loaded for the language `#1'. Using the pattern for}%
\typeout{** the default language instead.}%
\else
\language=\csname l@#1\endcsname
\fi
#2}}
\providecommand{\BIBdecl}{\relax}
\BIBdecl

\bibitem{bircher2016three}
A.~Bircher, M.~Kamel, K.~Alexis, M.~Burri, P.~Oettershagen, S.~Omari, T.~Mantel, and R.~Siegwart, ``Three-dimensional coverage path planning via viewpoint resampling and tour optimization for aerial robots,'' \emph{Autonomous Robots}, vol.~40, no.~6, pp. 1059--1078, 2016.

\bibitem{mcguire2016bridge}
B.~McGuire, R.~Atadero, C.~Clevenger, and M.~Ozbek, ``Bridge information modeling for inspection and evaluation,'' \emph{Journal of Bridge Engineering}, vol.~21, no.~4, p. 04015076, 2016.

\bibitem{chan2015towards}
B.~Chan, H.~Guan, J.~Jo, and M.~Blumenstein, ``Towards {UAV}-based bridge inspection systems: A review and an application perspective,'' \emph{Structural Monitoring and Maintenance}, vol.~2, no.~3, pp. 283--300, 2015.

\bibitem{fu2019toward}
M.~Fu, A.~Kuntz, O.~Salzman, and R.~Alterovitz, ``Toward asymptotically-optimal inspection planning via efficient near-optimal graph search,'' \emph{Robotics science and systems: online proceedings}, vol. 2019, 2019.

\bibitem{bircher2017incremental}
A.~Bircher, K.~Alexis, U.~Schwesinger, S.~Omari, M.~Burri, and R.~Siegwart, ``An incremental sampling-based approach to inspection planning: the rapidly exploring random tree of trees,'' \emph{Robotica}, vol.~35, no.~6, pp. 1327--1340, 2017.

\bibitem{sun2014wifi}
Y.~Sun, M.~Liu, and M.~Q.-H. Meng, ``Wifi signal strength-based robot indoor localization,'' in \emph{2014 IEEE International Conference on Information and Automation (ICIA)}.\hskip 1em plus 0.5em minus 0.4em\relax IEEE, 2014, pp. 250--256.

\bibitem{rida2015indoor}
M.~E. Rida, F.~Liu, Y.~Jadi, A.~A.~A. Algawhari, and A.~Askourih, ``Indoor location position based on bluetooth signal strength,'' in \emph{2015 2nd International Conference on Information Science and Control Engineering}.\hskip 1em plus 0.5em minus 0.4em\relax IEEE, 2015, pp. 769--773.

\bibitem{klein2011vehicle}
I.~Klein, S.~Filin, and T.~Toledo, ``Vehicle constraints enhancement for supporting ins navigation in urban environments,'' \emph{NAVIGATION, Journal of the Institute of Navigation}, vol.~58, no.~1, pp. 7--15, 2011.

\bibitem{JansonSP15}
L.~Janson, E.~Schmerling, and M.~Pavone, ``Monte carlo motion planning for robot trajectory optimization under uncertainty,'' in \emph{International Symposium of Robotics Research (ISRR)}, vol.~3.\hskip 1em plus 0.5em minus 0.4em\relax Springer, 2015, pp. 343--361.

\bibitem{MelchiorS07}
N.~A. Melchior and R.~G. Simmons, ``Particle {RRT} for path planning with uncertainty,'' in \emph{International Conference on Robotics and Automation (ICRA)}, 2007, pp. 1617--1624.

\bibitem{BergPA12}
J.~van~den Berg, S.~Patil, and R.~Alterovitz, ``Motion planning under uncertainty using iterative local optimization in belief space,'' \emph{Int. J. Robotics Res.}, vol.~31, no.~11, pp. 1263--1278, 2012.

\bibitem{DBLP:journals/ijrr/FuKSA23}
M.~Fu, A.~Kuntz, O.~Salzman, and R.~Alterovitz, ``Asymptotically optimal inspection planning via efficient near-optimal search on sampled roadmaps,'' \emph{Int. J. Robotics Res.}, vol.~42, no. 4-5, pp. 150--175, 2023.

\bibitem{papachristos2019autonomous}
C.~Papachristos, M.~Kamel, M.~Popovi{\'c}, S.~Khattak, A.~Bircher, H.~Oleynikova, T.~Dang, F.~Mascarich, K.~Alexis, and R.~Siegwart, ``Autonomous exploration and inspection path planning for aerial robots using the robot operating system,'' in \emph{Robot Operating System (ROS)}.\hskip 1em plus 0.5em minus 0.4em\relax Springer, 2019, pp. 67--111.

\bibitem{hazra2017using}
A.~Hazra, ``Using the confidence interval confidently,'' \emph{Journal of thoracic disease}, vol.~9, no.~10, p. 4125, 2017.

\bibitem{pepy2006safe}
R.~Pepy and A.~Lambert, ``Safe path planning in an uncertain-configuration space using rrt,'' in \emph{2006 IEEE/RSJ International Conference on Intelligent Robots and Systems}.\hskip 1em plus 0.5em minus 0.4em\relax IEEE, 2006, pp. 5376--5381.

\bibitem{alami1994planning}
R.~Alami and T.~Simeon, ``Planning robust motion strategies for a mobile robot,'' in \emph{Proceedings of the 1994 IEEE International Conference on Robotics and Automation}.\hskip 1em plus 0.5em minus 0.4em\relax IEEE, 1994, pp. 1312--1318.

\bibitem{candido2010minimum}
S.~Candido and S.~Hutchinson, ``Minimum uncertainty robot path planning using a pomdp approach,'' in \emph{2010 IEEE/RSJ International Conference on Intelligent Robots and Systems}.\hskip 1em plus 0.5em minus 0.4em\relax IEEE, 2010, pp. 1408--1413.

\bibitem{delamer2019solving}
J.-A. Delamer, Y.~Watanabe, and C.~P.~Carvalho~Chanel, ``Solving path planning problems in urban environments based on a priori sensors availabilities and execution error propagation,'' in \emph{AIAA Scitech 2019 Forum}, 2019, p. 2202.

\bibitem{englot2016sampling}
B.~Englot, T.~Shan, S.~D. Bopardikar, and A.~Speranzon, ``Sampling-based min-max uncertainty path planning,'' in \emph{2016 IEEE 55th Conference on Decision and Control (CDC)}.\hskip 1em plus 0.5em minus 0.4em\relax IEEE, 2016, pp. 6863--6870.

\bibitem{Wu.ea.22}
A.~Wu, T.~Lew, K.~Solovey, E.~Schmerling, and M.~Pavone, ``Robust-rrt: Probabilistically-complete motion planning for uncertain nonlinear systems,'' in \emph{International Foundation of Robotics Research}, 2022.

\bibitem{zheng2023ibbt}
D.~Zheng and P.~Tsiotras, ``Ibbt: Informed batch belief trees for motion planning under uncertainty,'' \emph{arXiv preprint arXiv:2304.10984}, 2023.

\bibitem{DBLP:conf/icra/HoSL22}
Q.~H. Ho, Z.~N. Sunberg, and M.~Lahijanian, ``Gaussian belief trees for chance constrained asymptotically optimal motion planning,'' in \emph{International Conference on Robotics and Automation}.\hskip 1em plus 0.5em minus 0.4em\relax {IEEE}, 2022, pp. 11\,029--11\,035.

\bibitem{DBLP:journals/trob/PedramFT23}
A.~R. Pedram, R.~Funada, and T.~Tanaka, ``Gaussian belief space path planning for minimum sensing navigation,'' \emph{{IEEE} Trans. Robotics}, vol.~39, no.~3, pp. 2040--2059, 2023.

\bibitem{galceran2013survey}
E.~Galceran and M.~Carreras, ``A survey on coverage path planning for robotics,'' \emph{Robotics and Autonomous systems}, vol.~61, no.~12, pp. 1258--1276, 2013.

\bibitem{danner2000randomized}
T.~Danner and L.~E. Kavraki, ``Randomized planning for short inspection paths,'' in \emph{Proceedings 2000 ICRA. Millennium Conference. IEEE International Conference on Robotics and Automation. Symposia Proceedings (Cat. No. 00CH37065)}, vol.~2.\hskip 1em plus 0.5em minus 0.4em\relax IEEE, 2000, pp. 971--976.

\bibitem{englot2010inspection}
B.~Englot and F.~Hover, ``Inspection planning for sensor coverage of 3d marine structures,'' in \emph{2010 IEEE/RSJ International Conference on Intelligent Robots and Systems}.\hskip 1em plus 0.5em minus 0.4em\relax IEEE, 2010, pp. 4412--4417.

\bibitem{englot2012sampling}
B.~J. Englot and F.~S. Hover, ``Sampling-based coverage path planning for inspection of complex structures,'' in \emph{Twenty-Second International Conference on Automated Planning and Scheduling}, 2012.

\bibitem{papadopoulos2013asymptotically}
G.~Papadopoulos, H.~Kurniawati, and N.~M. Patrikalakis, ``Asymptotically optimal inspection planning using systems with differential constraints,'' in \emph{2013 IEEE International Conference on Robotics and Automation}.\hskip 1em plus 0.5em minus 0.4em\relax IEEE, 2013, pp. 4126--4133.

\bibitem{BircherKAOS18}
A.~Bircher, M.~Kamel, K.~Alexis, H.~Oleynikova, and R.~Siegwart, ``Receding horizon path planning for 3d exploration and surface inspection,'' \emph{Auton. Robots}, vol.~42, no.~2, pp. 291--306, 2018.

\bibitem{PapachristosMKD19}
C.~Papachristos, F.~Mascarich, S.~Khattak, T.~Dang, and K.~Alexis, ``Localization uncertainty-aware autonomous exploration and mapping with aerial robots using receding horizon path-planning,'' \emph{Auton. Robots}, vol.~43, no.~8, pp. 2131--2161, 2019.

\bibitem{bookGroves}
P.~Groves, \emph{Principles of GNSS, Inertial, and Multisensor Integrated Navigation Systems, Second Edition}, 03 2013.

\bibitem{gross2015robust}
J.~N. Gross, Y.~Gu, and M.~B. Rhudy, ``Robust uav relative navigation with dgps, ins, and peer-to-peer radio ranging,'' \emph{IEEE Transactions on Automation Science and Engineering}, vol.~12, no.~3, pp. 935--944, 2015.

\bibitem{khaghani2016autonomous}
M.~Khaghani and J.~Skaloud, ``Autonomous vehicle dynamic model-based navigation for small uavs,'' \emph{NAVIGATION: Journal of the Institute of Navigation}, vol.~63, no.~3, pp. 345--358, 2016.

\bibitem{LaValle2006_Book}
S.~M. LaValle, \emph{{Planning Algorithms}}.\hskip 1em plus 0.5em minus 0.4em\relax Cambridge, U.K.: Cambridge University Press, 2006.

\bibitem{Salzman19}
O.~Salzman, ``Sampling-based robot motion planning,'' \emph{Commun. {ACM}}, vol.~62, no.~10, pp. 54--63, 2019.

\bibitem{Hart1968_TSSC}
P.~E. Hart, N.~J. Nilsson, and B.~Raphael, ``{A formal basis for the heuristic determination of minimum cost paths},'' \emph{IEEE Trans. Systems Science and Cybernetics}, vol.~4, no.~2, pp. 100--107, 1968.

\bibitem{Karaman2011_IJRR}
S.~Karaman and E.~Frazzoli, ``{Sampling-based algorithms for optimal motion planning},'' \emph{Int. J. Robotics Research}, vol.~30, no.~7, pp. 846--894, Jun. 2011.

\bibitem{neyman1937outline}
J.~Neyman, ``Outline of a theory of statistical estimation based on the classical theory of probability,'' \emph{Philosophical Transactions of the Royal Society of London. Series A, Mathematical and Physical Sciences}, vol. 236, no. 767, pp. 333--380, 1937.

\bibitem{habtzghi2014modified}
D.~Habtzghi, C.~Midha, and A.~Das, ``Modified clopper-pearson confidence interval for binomial proportion.'' \emph{J. Stat. Theory Appl.}, vol.~13, no.~4, pp. 296--310, 2014.

\bibitem{weisstein2004bonferroni}
E.~W. Weisstein, ``Bonferroni correction,'' \emph{https://mathworld. wolfram. com/}, 2004.

\bibitem{loh1996latin}
W.-L. Loh, ``On latin hypercube sampling,'' \emph{The annals of statistics}, vol.~24, no.~5, pp. 2058--2080, 1996.

\bibitem{knusel1986computation}
L.~Kn{\"u}sel, ``Computation of the chi-square and poisson distribution,'' \emph{SIAM Journal on Scientific and Statistical Computing}, vol.~7, no.~3, pp. 1022--1036, 1986.

\bibitem{games1977improved}
P.~A. Games, ``An improved t table for simultaneous control on g contrasts,'' \emph{Journal of the American Statistical Association}, vol.~72, no. 359, pp. 531--534, 1977.

\bibitem{grasmair2016basic}
M.~Grasmair, ``Basic properties of convex functions,'' \emph{Department of Mathematics, Norwegian University of Science and Technology}, 2016.

\end{thebibliography}


\appendices

\section{Statistical Background}
\label{app:stat}

In this paper we use CI to evaluate two different types of quantities with respect to the behavior of the \irisuu algorithm. The first type is the probability of success for some event, such as the collision probability of a given path being above a certain value.
The second type is the mean of a population, such as the expectation of a path's length.
Calculating the CI for these two quantities is done differently and we now detail each method:

\paragraph{Probability of success}
To calculate the CI of the success probability of a random variable, we  use the Clopper-Pearson method~\cite{habtzghi2014modified}. This method evaluates the maximum likelihood of the probability and its CI assuming a binomial distribution given finite independent trials. 
Specifically, let~$X$ be a random variable whose true unknown probability of success is~$\bar{p}$. 
Let~$x_1, \ldots, x_m$ be the outcome of~$m$ samples drawn from~$X$ (i.e.,~$x_i \in \{0, 1 \}$)
and let~$\hat{p}:= \frac{1}{m} \sum_{i=1}^{i=m} {x_i}$ be the estimated success probability.
Then, according to the Clopper-Pearson method for any~$\alpha \in [0,1]$, we can say with CL of~$1-\alpha$ that~$\bar{p}$, the true unknown success probability of~$X$, is within the following the CI:
\begin{equation}
\bar{p} \in \left[ \hat{p}^-(\hat{p},m,\alpha), \hat{p}^+(\hat{p},m,\alpha) \right],    
\end{equation}
where,
  \begin{subequations}\label{eq:Clopper-Pearson method}
     \begin{align}
            \hat{p}^-(\hat{p},m,\alpha) & :=  
                \frac{
                    \hat{p} \cdot F^{-1}\left( 2\hat{p},2  \lambda^-_{\hat{p},m} ,1-\frac{\alpha}{2} \right)}{
                    \lambda^-_{\hat{p},m}+\hat{p} \cdot F^{-1}\left( 2\hat{p},2\lambda^-_{\hat{p},m},1-\frac{\alpha}{2} \right)},\\
            \hat{p}^+(\hat{p},m,\alpha) &:= 
                \frac{
                    \hat{p}' \cdot F^{-1}\left( 2\hat{p}',2\lambda^+_{\hat{p},m},\frac{\alpha}{2} \right)}
                                {\lambda^+_{\hat{p},m}+\hat{p}' \cdot F^{-1}\left( 2\hat{p}',2\lambda^+_{\hat{p},m},\frac{\alpha}{2} \right)}.
    \end{align}
\end{subequations}
Here,~$F^{-1}$ is the inverse F-distribution function~\cite{knusel1986computation}, 
$\lambda^-_{\hat{p},m}:=m - \hat{p} + 1~$,
$\lambda^+_{\hat{p},m}:=m - \hat{p}~$
and
$\hat{p}' := \hat{p} + 1$.
\paragraph{Mean of a population}
To calculate the CI of the mean of a population, we follow Habtzghi et al.~\cite{habtzghi2014modified}.
Specifically, let~$X$ be some random process with unknown mean value~$\bar{X}$.
In addition, let~$x_1, \ldots, x_m$ be~$m$ be samples drawn from~$X$ and let~$\hat{X}:= \frac{1}{m}\sum_{i=1}^{i=m} {x_i}$ and~$\hat{s}:=\sqrt{\frac{1}{m-1}\sum_{i=1}^{i=m}{\left( x_i - \hat{x} \right)^2}}$ be their estimated mean and standard deviation, respectively.
Then, for any~$\alpha \in [0,1]$ we can say with CL of~$1-\alpha$ that~$\bar{X}$, the true unknown mean value of~$X$, is within the following CI:
\begin{equation}
\label{eq:Mean of a population-base}
\bar{X} \in \left[\bar{X}^-(\hat{X},m,\alpha), \bar{X}^+(\hat{X},m,\alpha) \right].    
\end{equation}
Where,
  \begin{subequations}\label{eq:Mean of a population}
\begin{align}
    \bar{X}^-(\hat{X},m,\alpha) := \hat{X} -  t^*\frac{\hat{s}}{\sqrt{m}}, \\
    \bar{X}^+(\hat{X},m,\alpha) := \hat{X} +  t^*\frac{\hat{s}}{\sqrt{m}}.
\end{align}
\end{subequations}
Here,~$t^*$ is computed according to standard t-tables~\cite{games1977improved}. 

\section{Proofs}

\label{app:Lemma-proofs}
We provide proofs for our lemmas.


\myproof{Lemma~\ref{lemma:Bounding executed path's expected coverage}}
We treat the executed path's coverage probability as a random variable and recall that~$\hat{p}_j^{\pi}$ is the estimated probability to inspect POI~$j$ computed by~$m$ independent samples of execution paths. 
Then, for any desired CL of~$1 - \alpha$, the lower bound on the coverage is defined as $\hat{p}{_j^{\pi}}^- =  \hat{p}^-(\hat{p}_j^{\pi},m,\alpha)$.
Here, the function~$\hat{p}^-$ is defined in Eq.~\eqref{eq:Clopper-Pearson method}.
Since we assume that the inspection of each POI is independent of the inspection outcomes of the other POIs~(see Assumption.~\ref{ass:iid}),
we can add up the lower bounds of the individual POIs to obtain a lower bound on the executed path's coverage.
Namely, we can say with a CL of {at least}~$1-\alpha$, a lower bound on the executed path's coverage is:
\begin{equation}
        \vert \bar{\S}(\pi) \vert^-  
        :=  \sum_{j=1}^{j=k}  \hat{p}{_j^{\pi}}^- 
        = \sum_{j=1}^{j=k} (\hat{p}_j^{\pi},m,\alpha).
\end{equation}
\myendproof

\begin{figure*}[tb]
\begin{subfigure}[c]{.24\textwidth}
  \includegraphics[width=\textwidth]{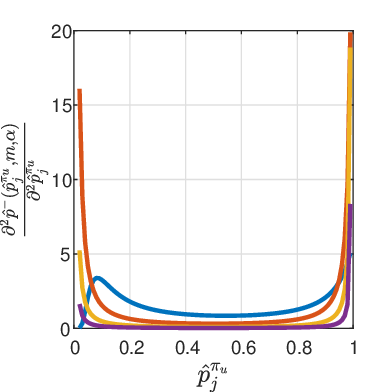}  
  \caption{$\alpha = 0.1$}
  \label{fig:sub-first}
\end{subfigure}
\begin{subfigure}[d]{.24\textwidth}
  \includegraphics[width=\textwidth]{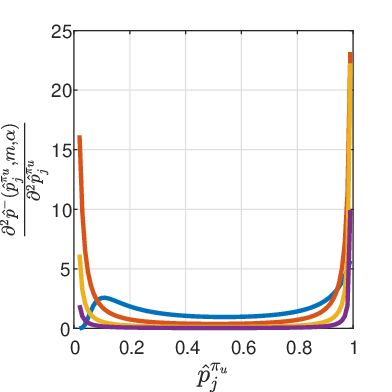}  
  \caption{$\alpha = 0.05$}
  \label{fig:sub-second}
\end{subfigure}
\begin{subfigure}[c]{.24\textwidth}
  \includegraphics[width=\textwidth]{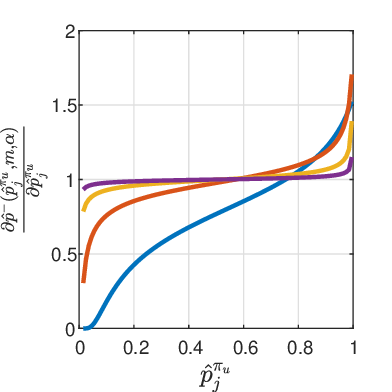}  
  \caption{$\alpha = 0.1$}
  \label{fig:sub-third}
\end{subfigure}
\begin{subfigure}[d]{.24\textwidth}
  \includegraphics[width=\textwidth]{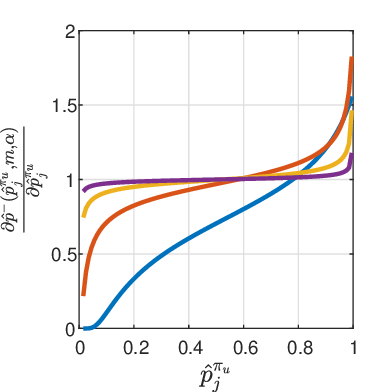}  
  \caption{$\alpha = 0.05$}
  \label{fig:sub-fourth}
\end{subfigure}

   \vskip\baselineskip
\centering
    \includegraphics[width=0.9\textwidth]{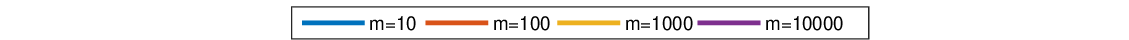}
\caption{
    Numerical demonstration that Eq.~\eqref{eq:assum-req1} 
    (%
    \protect \subref{fig:sub-first} 
    and
    \protect \subref{fig:sub-second}%
    )
    and
    Eq.~\eqref{eq:assum-req2}
    (%
    \protect \subref{fig:sub-third} 
    and
    \protect \subref{fig:sub-fourth}%
    )
    hold for different values of $\alpha$ and $m$.
}
\label{fig:numerically_p_plus}
\end{figure*}

We now proceed to address the next four lemmas,  whose proofs are straightforward. For Lemma~\ref{lemma:Bounding executed path's collision}, note that since we can treat the executed path's collision probability as a random variable,
Eq.~\eqref{eq:lemma2-collision} is an immediate application of the Clopper-Pearson method~(see Eq.~\eqref{eq:Clopper-Pearson method}). Similarly, for Lemma~\ref{lemma:Bounding executed path's expected length} we treat the executed path's expected length as random variable and 
Eq.~\eqref{eq:CI-length-lemma1} follows from  Eq.~\eqref{eq:Mean of a population-base}. 
\ignoreStat{
For Lemma~\ref{lemma:Bounding executed path's expected variance length}, Eq.~\eqref{eq:CI-length-variance} is an immediate application of Eq.~\eqref{eq:variance_bound_Sheskin}. 
Finally, for Lemma~\ref{lemma:Bounding executed path's expected possible length}, Eq.~\eqref{eq:CI_sigma_level_possible} immediately follows from Eq.~\eqref{eq:CI_sigma_level}.
}
We proceed to the final proof.

\myproof{Lemma~\ref{lemma:Bounding executed path's sub-optimal coverage}}
Considering that both values~$m$ and~$\alpha$ are fixed, it would be convenient to define:
\begin{equation}
    f(x): = \hat{p}^-(x,m,\alpha).
\end{equation}

Assume that minimization of $\sum_{j=1}^{j=k} f(x_j)$ is achieved for some values~$x_1^*, \ldots, x_k^*$ where $ x_j^* \in [0,1]$ and $\sum_{j=1}^{j=k} x_j^* \geq k \cdot \kappa$.
We need to show that $\forall j, x_j^* = \kappa$.
Finally, recall that $\kappa \in [0,1]$.
Thus, we distinguish between the cases where $\kappa = 1$ and $\kappa < 1$.

\noindent
\textbf{Case 1 ($\kappa = 1$)}:
Here, $\sum_{j=1}^{j=k} x_j^* \geq k \cdot \kappa  = k$. 
As $x_j^* \in [0,1]$, it follows that $\forall j, x_j^* = 1$.

\noindent
\textbf{Case 2 ($\kappa < 1$)}:
In this case, the proof will be done in two steps:
In step 1, we show that minimization of $\sum_{j=1}^{j=k} f(x_j)$ is achieved when all values of $x_j^*$ are equal.
Namely, $\exists \kappa' \geq \kappa$ s.t. $\forall j, x_j^* = \kappa'$.
In step 2, we show that this value is obtained for $\kappa' = \kappa$
These steps require that Assumption~\ref{ass:mon-convex} holds. 
Namely, that $f(x)$ is convex and monotonically increasing.

\noindent
\textbf{Case 2, step 1}:
Let $\kappa' \in [\kappa, 1)$ be a constant such that $\sum_{j=1}^{j=k} x_j^* = k \cdot \kappa'$
(note that such $\kappa'$ always exists).
We will prove by contradiction that $\forall j, x_j^* = \kappa'$ which concludes this step.

W.l.o.g., $x_1^* > \kappa$, i.e., there exists $\delta_1 >0$ such that $x_1^* = \kappa + \delta_1$.
As $\sum_{j=1}^{j=k} x_j^* = k \cdot \kappa'$, there at least one $j\neq1$ s.t. $x_j^* < \kappa'$.
Additionally, w.l.o.g.\ $x_2^* < \kappa'$. 
Namely, there exists $\delta_2 >0$ such that $x_2^* = \kappa' - \delta_2$.

Let~$\delta = \min(\delta_1, \delta_2)$ and consider the following solution $x_1', \ldots, x_k' $ to our minimization problem:
$x_1' = x_1^* - \delta$, 
$x_2' = x_2^* + \delta$
and
$x_j' = x_j^*$ for $2<j\leq k$.
Notice that this is a valid solution (i.e.,~$\forall j, x_j' \in [0,1]$ and~$\sum_{j=1}^{j=k} x_j' \geq k \cdot \kappa$). 
We will show that $\sum_{j=1}^{j=k} f(x_j') < \sum_{j=1}^{j=k} f(x_j^*)$ which will lead to a contradiction that minimizing~$\sum_{j=1}^{j=k} f(x_j)$ is  achieved for  $x_1^*, \ldots, x_k^*$.

As $x_j' = x_j^*$ for $2<j\leq k$, to show that $\sum_{j=1}^{j=k} f(x_j') < \sum_{j=1}^{j=k} f(x_j^*)$, it suffices to prove that:
\begin{equation} 
f(x_1') + f(x_2') < f(x_1^*) + f(x_2^*).
\end{equation}
This will be done using Assumption~\ref{ass:mon-convex} (i.e., that $f(\cdot)$ is a strictly convex function).
Namely, that~$\frac{\partial^2 f(x)}{\partial^2 x} >0,\forall x \in [0,1]$
which implies (see, e.g.,~\cite{grasmair2016basic}) that~$\forall \lambda$ s.t. $0 < \lambda < 1$  we have:
\begin{equation}\label{eq:convex_property_1}
 f(\lambda x_1 + (1-\lambda) x_2) < \lambda f(x_1) + (1-\lambda) f(x_2).
\end{equation}
Thus,
\begin{equation} 
\begin{split}
    f(x_1') + f(x_2') 
        & = f(x_1^* - \delta) + f(x_2^* + \delta) \\
        & \stackrel{(a)}{=} f(x_1^* - \lambda (x_1^* - x_2^*)) + f(x_2^* + \lambda (x_1^* - x_2^*))\\
        & = f(\lambda x_2^* + (1-\lambda)x_1^*) + f(\lambda x_1^* + (1-\lambda)  x_2^*)\\
        & \stackrel{(b)}{<} (1-\lambda) f(x_1^*) +\lambda f(x_2^*) + \\
           &~~~~ \lambda f(x_1^*) + (1-\lambda) f(x_2^*) \\
        & = f(x_1^*) + f(x_2^*).
\end{split}
\end{equation}
In (a) we use the equation $\lambda= \delta / (x_1^* - x_2^*)$ and note that~$0 < \lambda < 1$. (b) follows from Eq.~\eqref{eq:convex_property_1}.

\noindent
\textbf{Case 2, step 2}:
We will prove this step by contradiction.
Assume that minimizing~$\sum_{j=1}^{j=k} f(x_j)$ is achieved when $\sum_{j=1}^{j=k} x_j^* > k \cdot \kappa$.
Namely,~$\exists \delta >0$ s.t. $\sum_{j=1}^{j=k} x_j^* = k \cdot( \kappa + \delta)$.
Following step 1, $\forall j, x_j^* = \kappa + \delta$ and~$\sum_{j=1}^{j=k} f(x_j^*) = k \cdot f(\kappa + \delta)$.

However, setting $x_j^* = \kappa$ is also a valid solution  for which~$\sum_{j=1}^{j=k} f(x_j^*) = k \cdot f(\kappa)$. 
Following Assumption~\ref{ass:mon-convex}, $f(\cdot)$ is a monotonically increasing function which leads to a contradiction.

\myendproof

As a supplement to the proof of Lemma~\ref{lemma:Bounding executed path's sub-optimal coverage} we numerically demonstrate that Assumption~\ref{ass:mon-convex} holds.
Specifically, recall that we used $f(x)$ to denote $\hat{p}^-(x,m,\alpha)$.
Thus, to show that $f(\cdot)$ monotonically increases and is a strictly convex function we wish to show that $\forall \hat{p}_j^{\pi_u} \in [0,1]$ both:
\begin{equation}
\label{eq:assum-req1}
  \frac{\partial \hat{p}^-(\hat{p}_j^{\pi_u},m,\alpha)}{\partial \hat{p}_j^{\pi_u}}>0,  
\end{equation}
and
\begin{equation}
\label{eq:assum-req2}
    \frac{\partial^2 \hat{p}^-(\hat{p}_j^{\pi_u},m,\alpha)}{\partial^2 \hat{p}_j^{\pi_u}} >0.
\end{equation}
This is demonstrated in Fig.~\ref{fig:numerically_p_plus} for different values of $m$ and~$\alpha$.

\section{Illustrative example for possible false negatives}
\label{subsec:Illustrative example for possible false negatives}
In this section, we present an illustrative example that serves to elucidate the notion of false negatives eluded to in Sec.\ref{sec:Theoretical guarantees}. The purpose of this example is to shed light on how the use of statistical guarantees, such as those outlined in Sec.\ref{subsec:Guarantees for a test path}, can potentially lead to erroneous conclusions when applied to the output of the \irisuu algorithm under specific conditions.

Specifically, consider a scenario in which \irisuu is configured with $\rho_{\rm coll} = 0$, meaning that if the algorithm outputs a path $\pi$, it confidently asserts that $\hat{C}(\pi) = 0$, indicating that the path is collision-free.

Furthermore, assume that the algorithm uses~$m=120$. As established in Lemma~\ref{lemma:Bounding executed path's collision}, when considering a given path $\pi$ and setting $\alpha = 0.05$, we find that $\bar{C}(\pi)^+ \approx 0.03$. This implies a~$95\%$ probability that if this path is executed~$100$ times, at most~$3$ of these executions will result in collisions.

Now, let's examine a situation where the false-negative bias comes into play — when there is more than one option for the command path. 
For instance, assume we have~$k=100$ paths, denoted as~$\pi_1, \ldots, \pi_k$, connecting the start and the goal where each of these paths has a true collision probability of~$\bar{C}(\pi_i) = 0.04$.
The probability that a specific path will be estimated to be collision-free is:
$$
(1-\bar{C}(\pi_i))^m =0.96^{120} \approx 0.0075.
$$
However, the false-negative bias effect lies in the probability that at least one of the $k$ paths will be estimated as collision-free:
$$
1 - \left(1 - (\bar{C}(\pi_i))^m \right)^k
\approx
1 - \left(1 - 0.0075\right)^{100}
\approx
0.52.
$$
Namely, there is more than $50\%$ chance that the algorithm will output a path $\pi_i$ assumed to be  collision-free whose true collision probability is $4\%$
(which, of course, is larger than the upper bound of $\bar{C}(\pi_i)^+ \approx 0.03$ guaranteed with $95\%$ confidence if Lemma~\ref{lemma:Bounding executed path's collision} was wrongly used).

\end{document}